\pdfoutput=1
\documentclass[10pt]{article}

\usepackage[letterpaper,margin=1in]{geometry}

\usepackage[numbers,sort&compress]{natbib}

\usepackage{lmodern}        
\usepackage[utf8]{inputenc}
\usepackage[T1]{fontenc}
\usepackage{hyperref}
\usepackage{url}
\usepackage{booktabs}
\usepackage{amsfonts}
\usepackage{nicefrac}
\usepackage{microtype}
\usepackage{xcolor}

\usepackage{graphicx}
\usepackage{multirow}
\usepackage{amsmath}
\usepackage{amssymb}
\usepackage{mathtools}
\usepackage{amsthm}
\usepackage{algorithm}
\usepackage{algorithmic}
\usepackage[capitalize,noabbrev]{cleveref}
\usepackage{placeins}
\usepackage{float}

\theoremstyle{plain}
\newtheorem{theorem}{Theorem}[section]
\newtheorem{proposition}[theorem]{Proposition}

\theoremstyle{definition}

\theoremstyle{remark}
\newtheorem{remark}[theorem]{Remark}

\title{RAID: Semantic Graph Diffusion for True Cold-Start and Cross-Lingual Forecasting}

\author{
  Arunkumar V$^{1}$ \quad Manoranjan Gandhudi$^{2}$ \quad Gangadharan G.~R.$^{3}$ \\[0.3em]
  Arun Prakash$^{4}$ \quad S.~Senthilkumar$^{1}$ \\[0.7em]
  $^{1}$University College of Engineering, Anna University Tiruchirappalli, Tamil Nadu, India \\
  $^{2}$Central University of Karnataka, India \\
  $^{3}$National Institute of Technology Tiruchirappalli, India \\
  $^{4}$Jawaharlal Nehru University, New Delhi, India \\[0.7em]
  \texttt{arunkumarv1530@gmail.com} \quad \texttt{gmanoranjan@cuk.ac.in} \quad \texttt{ganga@nitt.edu} \\
  \texttt{arunprakash@mail.jnu.ac.in} \quad \texttt{ssk@aubit.edu.in}
}

\begin{document}

\maketitle

\begin{abstract}
Time-series foundation models show strong transfer performance when given a non-empty history window. However, true cold-start scenarios, where a new item has no prior observations, violate this assumption. We propose RAID (Retrieval-Augmented Iterative Diffusion) a framework, which replaces history-based correlation learning with metadata-driven semantic retrieval and graph-conditioned diffusion. RAID maps textual metadata into a shared semantic space using a frozen multilingual embedding model and constructs an inductive retrieval graph that extends naturally to unseen items. It first forms a base forecast by aggregating information from semantically related neighbors, then refines this forecast with a gated diffusion module to model residual uncertainty. Under a strict true cold-start protocol, RAID outperforms strong foundation models and competitive baselines on both forecasting accuracy and prediction interval coverage, while reducing inference latency by an order of magnitude through non-autoregressive decoding. The shared semantic space also enables zero-shot cross-lingual transfer, allowing a model trained on English descriptions to generalize to items described in other languages without direct supervision.
\end{abstract}

\section{Introduction}
\label{sec:intro}

Time series forecasting has recently shifted from specialized architectures to general-purpose foundation models. Large-scale transformers such as Chronos~\cite{chronos} and TimesFM~\cite{timesfm} treat forecasting as a next-token prediction task and obtain strong zero-shot performance by scaling model size and training data. Every such architecture, however, depends on a historical context window $x_{t-L:t}$ to generate future predictions, and the quadratic $\mathcal{O}(L^2)$ cost of attention in $L$ forces a trade-off between accuracy and latency.

This dependence on history creates a failure mode in the true cold-start regime. Foundation models like TimesFM~\cite{timesfm} and Chronos~\cite{chronos} claim zero-shot capabilities, yet their published ``zero-shot'' setting still requires each test item to carry its own history window ($L > 0$)~\cite{gifteval,ttm,moirai}. The underlying mechanism is \textit{In-Context Forecasting} (ICF)~\cite{timesfmicf}, not reasoning from item description. When a new item has no observations ($L=0$), these models collapse to a global mean or random noise~\cite{oreshkin2021metalearning,graph_patching} and cannot exploit the rich metadata (e.g., product descriptions, server specs) that is available prior to the first observation~\cite{contextformer,cik,langtime}.

We introduce \textbf{RAID} (Retrieval-Augmented Iterative Diffusion), a framework that unifies semantic retrieval with generative diffusion. The inductive bias is \textit{semantic homophily}, the observation that items with similar metadata tend to exhibit similar demand dynamics, so a new ``Nike Running Shoe'' inherits seasonality from existing ``Adidas Running Shoes'' even when their sales histories never overlap. RAID realizes this by leveraging a frozen, multilingual Large Language Model (LLM) to inductively construct a retrieval graph from textual metadata, propagating demand signals from established ``source'' items to new ``target'' items. By decoupling the memory mechanism (retrieval) from the generation mechanism (diffusion), RAID also enables an efficient $\mathcal{O}(1)$ inference mode that bypasses the quadratic cost of transformers.

\subsection{Key Contributions}
Our primary contributions are as follows:

\begin{itemize}
    \item \textbf{Diffusion Transformer Architecture.} We propose RAID, a novel framework that employs a Diffusion Transformer (DiT) backbone. Unlike standard convolutional approaches, RAID leverages Adaptive Layer Normalization (AdaLN) to condition the generative process on static semantic graphs, enabling robust residual modeling.
    \item \textbf{Non-Autoregressive Efficiency.} By decoupling memory (retrieval) from generation, RAID pairs a precomputed Shape-Scale base prediction with entropy-reduced diffusion that converges in $S=10$ refinement steps, an order of magnitude below the $50$--$100$ typical of raw-data diffusion. The full generative variant (RAID-Full) achieves a $48.7\times$ inference speedup over Chronos-Small at matched generative capability, and a feed-forward retrieval-only variant (RAID-Base) reaches $124\times$ when probabilistic refinement is not required (details in Appendix~\ref{app:efficiency_details}).
    \item \textbf{Cross-Lingual Transfer.} We demonstrate that RAID generalizes across languages and scripts (Latin / Cyrillic / Hanzi). A model trained on English data achieves performance competitive with native training on Spanish and 1C (RU) datasets, and we hypothesize a structural-regularization mechanism driven by source-graph density.
    \item \textbf{Cold-Start Performance.} In the strict true cold-start regime ($L=0$), RAID achieves the best average rank across nine datasets and wins per-dataset sMAPE on six of them. Foundation models that rely on a history window collapse toward a global prior in this regime, while RAID propagates demand signal through the semantic graph and refines it with the diffusion residual.
\end{itemize}

\section{Related Work}

\subsection{Foundation Models and Cold-Start}
The paradigm of forecasting has shifted towards large-scale pre-training. Foundation models such as Chronos~\cite{chronos}, TimesFM~\cite{timesfm}, and Lag-Llama~\cite{lagllama} treat time series as a language modeling task, tokenizing values and applying transformer architectures to learn universal temporal patterns. While these models demonstrate impressive zero-shot transfer capabilities across datasets, they define ``zero-shot'' as prediction on unseen domains given a historical context window. They remain fundamentally autoregressive, requiring a sequence of past observations $x_{1:t}$ to generate $x_{t+1}$. This dependency renders them ineffective for true cold-start scenarios where $t=0$. 
Prior work on cold-start forecasting typically relies on (i) meta-learning across tasks,
(ii) factorization and embedding-based methods, or (iii) supervised regressors over static metadata~\cite{oreshkin2021metalearning,graph_patching,contextformer}.
While effective in limited settings, these approaches often struggle to represent the rich, non-linear structure available in free-form text and do not provide calibrated predictive uncertainty.

\textbf{Text-aligned forecasting and its limits.}
Approaches like LangTime~\cite{langtime} and Context-is-Key~\cite{cik} pair LLMs with time series via RL or in-context prompting. Both assume $L>0$: text augments history rather than replacing it, so their similarity and retention mechanisms collapse when the window is empty. RAID is, to our knowledge, the first method to use textual metadata as the \emph{sole} conditioning signal under a strict $L=0$ protocol. Direct LLM prompting in this regime exhibits \textit{scale hallucination}, often missing demand magnitude by orders of magnitude (e.g., predicting 1000 units for a niche item that sells 5), which motivates the shape-scale decoupling in Section~\ref{sec:retrieval}.

\subsection{The Limits of In-Context Forecasting}
Recent advancements in ICF, such as TimesFM-ICF~\cite{timesfmicf}, have attempted to bridge the gap between fixed-horizon models and flexible zero-shot tasks. These models typically employ a retrieval mechanism to populate the context window with related time series, thereby allowing the transformer to attend to relevant historical patterns. However, this approach relies on a fundamental assumption: that there exists a sufficient historical window $x_{1:t}$ to measure similarity against potential context examples. In the true cold-start setting where $t=0$, this similarity metric is undefined. Consequently, the context selection mechanism degrades to random sampling or reliance on global averages. As shown in our experiments, attending to randomly selected, semantically irrelevant trajectories introduces stochastic noise that often degrades performance below simple univariate baselines. RAID eliminates this dependency by grounding the retrieval process in static semantic metadata, ensuring that the context remains relevant even when temporal observations are entirely absent.

\subsection{Graph Neural Networks and Diffusion}
Graph Neural Networks (GNNs)~\cite{velickovic2018gat,brody2022how} have been widely adopted to model spatial correlations in multivariate time series~\cite{mtgnn}. However, standard approaches typically learn the adjacency matrix via gradient descent or historical correlations. These methods are transductive, meaning they cannot easily accommodate new nodes without retraining or historical data to estimate correlations.

Generative diffusion models, such as TimeGrad~\cite{timegrad} and CSDI~\cite{csdi}, have set new standards for probabilistic forecasting. Yet, these models typically condition the diffusion process on the time series history itself. RAID differs by conditioning the diffusion process on an \textit{inductively} constructed semantic graph using a Diffusion Transformer (DiT) backbone. By leveraging pre-trained LLM embeddings to define the graph topology, RAID extends the generative power of diffusion to the zero-shot regime, using semantic neighbors to guide the denoising process even when temporal history is absent.

\section{Method: RAID}
\label{sec:method}

\begin{figure*}[t]
\centering
\includegraphics[width=0.8\textwidth]{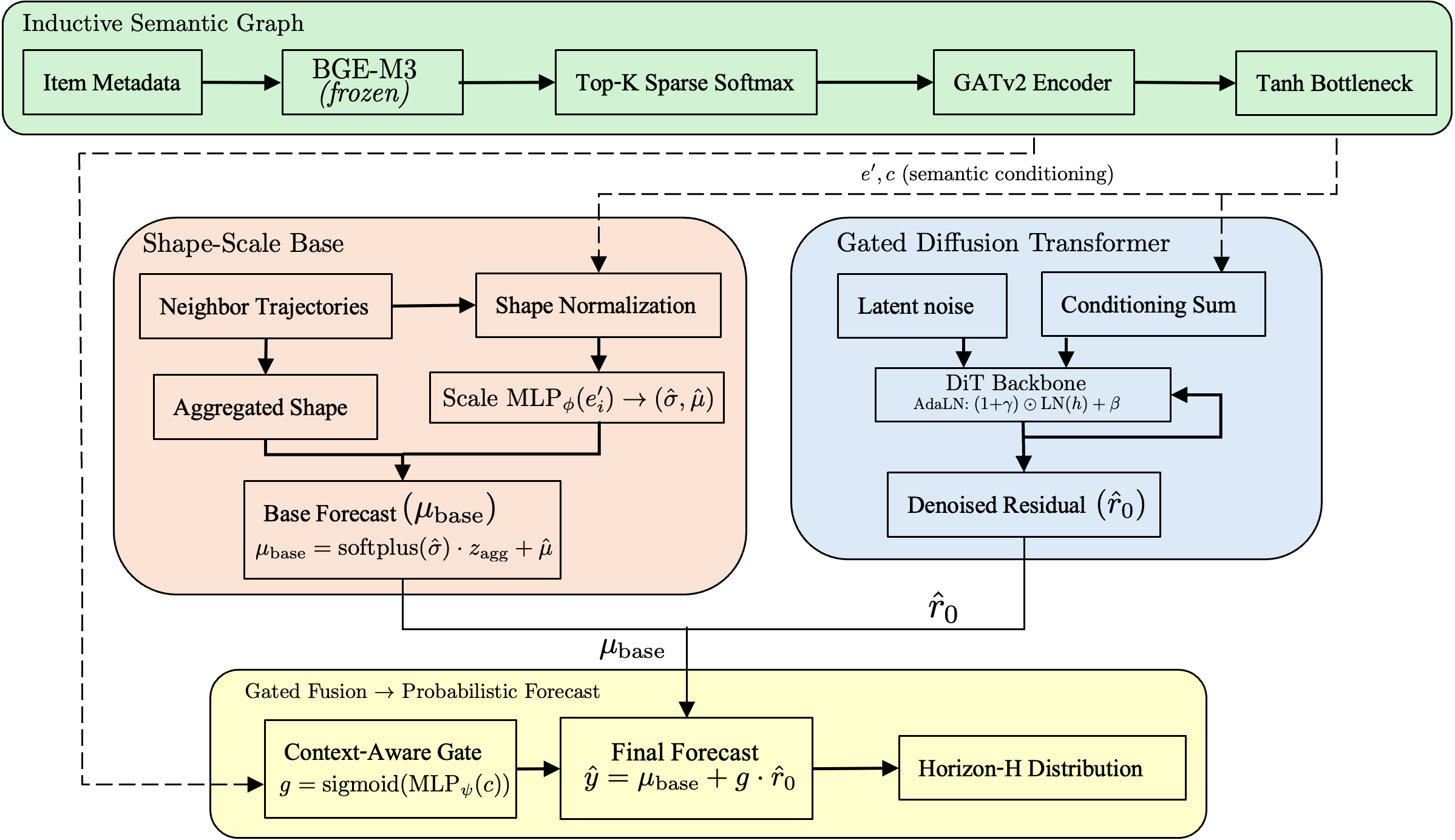}
\caption{\textbf{The RAID Architecture.} (Top) Inductive Semantic Graph: Metadata is mapped to a shared semantic manifold via BGE-M3 to construct a sparse retrieval graph. (Bottom Left) Shape-Scale Decomposition: Neighbor trajectories are normalized into scale-invariant shapes and aggregated; a target-specific scale is predicted via MLP to reconstruct a deterministic base forecast $\mu_{\text{base}}$. (Bottom Right) Gated Diffusion: A Diffusion Transformer (DiT) models the residual uncertainty $\mathbf{r}$, conditioned on the graph context $\mathbf{c}$ via Adaptive Layer Norm. The final forecast $\hat{y} = \mu_{\text{base}} + g \cdot \hat{r}_0$ is emitted as a probabilistic Horizon-$H$ distribution.}
\label{fig:main_arch}
\end{figure*}

\subsection{Problem Setting and Semantic Embedding Space}
We define the cold-start forecasting problem as learning a conditional density $p_\theta(\mathbf{y}_{t:t+H} \mid m_i, \mathcal{K})$, where $m_i$ is the static metadata for item $i$, $\mathcal{K}$ is the knowledge retrieval base, and historical observations $\mathbf{y}_{1:t}$ are unavailable ($\emptyset$) or negligible (Figure~\ref{fig:main_arch}). The core of RAID is the projection of heterogeneous time series into a shared semantic manifold. Let $m_i$ be the raw textual description of agent $i$ (e.g., ``Nike Air Zoom Pegasus, Mens Running Shoe''). We employ a frozen, multilingual embedding model, BGE-M3~\cite{bge}, to map this text to a dense vector $e_i \in \mathbb{R}^d$:

\begin{equation}
    e_i = \frac{\Phi_{\text{LLM}}(m_i)}{\|\Phi_{\text{LLM}}(m_i)\|_2}
\end{equation}
where $\Phi_{\text{LLM}}$ denotes the frozen BGE-M3 encoder and $\ell_2$-normalization places $e_i$ on the unit hypersphere, ensuring scale-invariant cosine similarity in subsequent retrieval.

Metadata is serialized with explicit field tokens, $m_i = \texttt{[CAT]}\,c_i \oplus \texttt{[BRAND]}\,b_i \oplus \texttt{[TITLE]}\,t_i \oplus \texttt{[DESC]}\,d_i$, so the graph topology reflects functional similarity rather than free-text overlap (full schema in Appendix~\ref{app:serialization}). The encoder is multilingual, so $e_i$ lies in a language-agnostic space and enables the cross-lingual transfer in Section~\ref{sec:experiments}.

The raw $1024$-dimensional embeddings are projected through a Tanh-activated bottleneck $e'_i = \mathrm{Tanh}(\mathrm{Linear}(\mathrm{LayerNorm}(e_i))) \in \mathbb{R}^{16}$ for use as conditioning features. Graph similarities are computed on $e_i$, and the learnable modules (Scale MLP, DiT) consume $e'_i$.

\subsection{Inductive Graph Topology}
\label{sec:graph}

Instead of learning a transductive adjacency matrix $\mathbf{A} \in \mathbb{R}^{N \times N}$ via backpropagation, which fails for unseen nodes, we construct an inductive graph structure $\mathcal{G} = (\mathcal{V}, \mathcal{E})$ derived strictly from the semantic manifold. 

To ensure the graph captures meaningful local neighborhoods while suppressing irrelevant long-tail connections, we employ a Top-$k$ sparse softmax kernel. For any node $i$, we restrict connections to its $k$ nearest neighbors $\mathcal{N}_k(i)$ in the embedding space. The adjacency weights are defined as:

\begin{equation}
    \label{eq:adjacency}
    A_{ij} = \frac{\exp(\langle e_i, e_j \rangle / \tau)}{\sum_{l \in \mathcal{N}_k(i)} \exp(\langle e_i, e_l \rangle / \tau)}
\end{equation}

where $\tau$ is a learnable temperature initialized to $1.0$.\footnote{In the implementation (Algorithm~\ref{alg:training}) we reparameterize $\tau \leftarrow \mathrm{Softplus}(\tau)$ to enforce positivity during gradient updates; the value plotted in $A_{ij}$ is always strictly positive.} This construction guarantees that the graph topology is dynamic and inductive; extending to a new node $u \notin \mathcal{V}$ simply requires a forward pass of the embedding model to identify its support set $\mathcal{N}_k(u)$.

\textbf{Graph Encoding (GATv2).}
To propagate information across this topology, we utilize a GATv2 encoder~\cite{brody2022how}. Unlike standard GATs which rely on static attention, GATv2 computes dynamic attention scores dependent on both the source and target node features, allowing the model to strictly attend to more relevant semantic neighbors regardless of their absolute embedding position.

\subsection{Retrieval-Augmented Base Prediction}
\label{sec:retrieval}

For a cold-start agent $i$ the autoregressive history $y_{i, 1:t}$ is unavailable. RAID retrieves a support set $\mathcal{N}_k(i)$ and forms a deterministic base prediction from neighbor dynamics. Averaging raw trajectories fails when neighbors share function but occupy different market tiers: on Amazon Electronics, the \textit{Echo Dot (3rd Gen)} and \textit{Echo Dot Kids Edition} have near-identical descriptions but differ by $112\times$ in weekly volume (Appendix~\ref{app:magnitude}). A naive average would overestimate demand for the niche variant by more than $100\times$.

We address this with a Shape-Scale Decomposition. Each retrieved trajectory $y_j$ is normalized into a scale-invariant shape $z_j$, with hard clamping to bound outliers:

\begin{equation}
    z_j(t) = \text{Clamp}\left(\frac{y_j(t) - \mu_j}{\sigma_j + \varepsilon}, -C, C\right)
\end{equation}

where $\mu_j$ and $\sigma_j$ are the global mean and standard deviation of neighbor $j$, $\varepsilon$ is a numerical stabilizer, and $C=10$ is a robustness constant derived from the data distribution. We then aggregate these shapes using the attention weights $A_{ij}$ derived in Eq.~\eqref{eq:adjacency} to form a normalized prior $z_{\text{agg}}$.

The final base prediction $\mu_{\text{base}}^{(i)}$ is reconstructed by modulating this shape with the target's specific scale $\hat{\sigma}_i$ and mean $\hat{\mu}_i$. These parameters are estimated via a Multi-Layer Perceptron (MLP) with GELU activations, conditioned on the bottleneck embedding $e'_i$:

\begin{align}
    (\hat{\sigma}_i, \hat{\mu}_i) &= \text{MLP}_\phi(e'_i) \\
    \mu_{\text{base}}^{(i)}(t) &= \text{Softplus}(\hat{\sigma}_i) \cdot z_{\text{agg}}(t) + \hat{\mu}_i
\end{align}

Softplus keeps scale estimates strictly positive. At inference, these values are precomputed into a hash table indexed by semantic cluster, giving $\mathcal{O}(1)$ retrieval.

\subsection{Gated Diffusion Transformer (DiT)}
\label{sec:diffusion}

To model the stochastic residual $\mathbf{r} = \mathbf{y} - \mu_{\text{base}}$, we employ a Denoising Diffusion Probabilistic Model (DDPM)~\cite{ho2020denoising}. While prior time-series diffusion models often rely on dilated convolutions, RAID adopts a Diffusion Transformer (DiT) backbone~\cite{peebles2023dit}. This architecture provides superior scalability and global context modeling through self-attention mechanisms.

\textbf{Conditioning via Adaptive Layer Normalization (AdaLN).}
A critical challenge is injecting the static semantic graph embedding $\mathbf{c}$ (derived from the GNN) into the dynamic denoising process. Standard cross-attention is computationally expensive for this purpose. Instead, we use AdaLN~\cite{perez2018film}. For a hidden state $h$ at layer $l$, the normalization parameters are regressed directly from the sum of the graph condition $\mathbf{c}$ and the diffusion step embedding $t_s$:

\begin{equation}
    \text{AdaLN}(h, \mathbf{c}, t_s) = (1 + \gamma_l(\mathbf{c}, t_s)) \odot \text{LN}(h) + \beta_l(\mathbf{c}, t_s)
\end{equation}

where $\gamma_l$ and $\beta_l$ are zero-initialized linear projections. This initialization ensures the conditioning starts inactive: AdaLN initially reduces to standard LayerNorm without context modulation, and the model learns to ``turn on'' the semantic guidance as training progresses.

\textbf{Gated Residual Fusion.}
To ensure stability in the true cold-start regime ($L=0$), we introduce a learnable gating mechanism to blend the deterministic base forecast with the stochastic residual. Unlike static weighting schemes, we employ a Context-Aware Gating Network parameterized as a shallow MLP. The final forecast is generated as:

\begin{equation}
    \hat{\mathbf{y}} = \mu_{\text{base}} + \operatorname{sigmoid}(\text{MLP}_\psi(\mathbf{c})) \cdot \hat{\mathbf{r}}_0
\end{equation}

where $\operatorname{sigmoid}(x) = (1+e^{-x})^{-1}$ and $\mathbf{c}$ is the context vector. The final bias of the MLP is initialized to $-3.0$. This initialization ($\operatorname{sigmoid}(-3.0) \approx 0.047$) imposes a structural bias, so that the model initially relies on the robust retrieval-based prior $\mu_{\text{base}}$ and dynamically learns to scale the residual contribution based on the specific uncertainty of the item context. A separate gate on the AdaLN context input that handles the warm-start ($L>0$) regime is described in Appendix~\ref{app:architecture_details}.

\begin{remark}[Residualization reduces target entropy]\label{prop:entropy}
By the law of total covariance, the residual target $R = X - \mathbb{E}[X \mid C]$ satisfies $\mathrm{Cov}(X) = \mathrm{Cov}(\mathbb{E}[X \mid C]) + \mathrm{Cov}(R)$, so whenever the conditional mean carries information ($\mathrm{Cov}(\mathbb{E}[X \mid C]) \neq 0$) the residual has strictly smaller log-determinant covariance and a strictly smaller Gaussian entropy upper bound than the raw target. A self-contained derivation is in Appendix~\ref{app:entropy_proof}.
\end{remark}

This entropy reduction motivates the small-$S$ diffusion regime: the residual distribution is materially simpler than the raw signal, and we observe empirically that $S = 10$ inference steps suffice to recover it, whereas standard data-space diffusion baselines (e.g., CSDI) typically require $S \ge 50$ (validated in Section~\ref{sec:efficiency} and the diffusion-step sweep in Appendix~\ref{app:efficiency_details}). We treat the connection between the theoretical entropy reduction and the empirical sampler-step count as an empirical finding rather than a formal convergence claim (training/inference pseudocode in Appendix~\ref{app:algorithms}).

\section{Experiments}
\label{sec:experiments}

\subsection{Datasets and Protocol}
We evaluate RAID on nine diverse real-world datasets characterized by varying degrees of metadata richness, ranging from unstructured descriptions (Amazon) to hierarchical categories (M5).

\textbf{Evaluation Protocol.} We employ a strict True Cold-Start evaluation. For the test set agents, the historical context window is masked ($x_{1:t} = \emptyset$). The model must rely solely on the inductive graph $\mathcal{G}$ and known future covariates (e.g., calendar embeddings) to forecast the horizon $H$. We report sMAPE for point accuracy and CRPS for probabilistic calibration.

\textbf{Training Protocol.} RAID trains on each dataset's training split and is evaluated on a held-out subset of cold-start items where target-item history is absent. Foundation models (Chronos, TimesFM, Moirai) are evaluated zero-shot, following the convention of TTM~\cite{ttm} and GIFT-Eval~\cite{gifteval}. Fine-tuning these foundation models on target-domain data does not close the cold-start gap because their published architectures lack a metadata-conditioning channel at $L=0$ (Appendix~\ref{app:fm_finetune}), so the comparison isolates the value of metadata-conditioned cold-start prediction rather than the amount of pre-training data. Implementation and compute details are in Appendix~\ref{app:implementation} and Appendix~\ref{app:compute}.

\subsection{Baselines}
We compare RAID against a rigorous suite representing the current state-of-the-art~\cite{ke2017lightgbm,das2023tide,lim2021tft,salinas2020deepar,moirai}. To ensure a fair comparison, all supervised baselines are trained using the same static metadata embeddings as RAID.

\begin{itemize}
    \item \textbf{Naive \& Classical:} Global Mean and LightGBM (gradient boosting on static features).
    \item \textbf{Supervised Deep Learning:}
    \begin{itemize}
        \item \textbf{DeepMF:} Deep Matrix Factorization initialized with semantic centroids.
        \item \textbf{TiDE \& TFT:} State-of-the-art transformer architectures modified to accept static embedding inputs.
        \item \textbf{DeepAR:} Probabilistic autoregressive RNN, trained with category embeddings.
    \end{itemize}
    \item \textbf{Foundation Models (Zero-Shot):} Chronos-Small, TimesFM, and Moirai-Base~\cite{moirai}. These are evaluated in a strict zero-shot setting ($L=0$) without fine-tuning.
    \item \textbf{Specialized Baselines:} 
    \begin{itemize}
        \item \textbf{CSDI:} Conditional Score-based Diffusion Imputation, using category means as the history input.
        \item \textbf{LLM-Direct:} Zero-shot prompting of GPT-4o with raw metadata descriptions (used to quantify scale hallucination).
    \end{itemize}
\end{itemize}

\begin{table*}[t]
\caption{\textbf{General forecasting protocol (full history available).} sMAPE ($\downarrow$) and CRPS ($\downarrow$) across nine datasets, with per-dataset rank in subscripts. LightGBM achieves the best average rank (2.1), reflecting the well-known strength of gradient-boosted trees on tabular forecasting with engineered lag features. Among neural methods, RAID is competitive with the foundation-model tier (3.9 vs.\ TimesFM/Chronos 3.8). The strict cold-start regime where RAID is designed to operate is reported in Table~\ref{tab:cold_start_specific}.}
\label{tab:main_results}
\centering
\tiny
\setlength{\tabcolsep}{2pt}
\renewcommand{\arraystretch}{1.25}
\resizebox{\textwidth}{!}{
\begin{tabular}{l | cc | cc | cc | cc | cc | cc | cc | cc | cc | c }
\toprule
\multirow{2}{*}{\textbf{Method}}
& \multicolumn{2}{c|}{\textbf{Amz Sports}} & \multicolumn{2}{c|}{\textbf{Job-SDF}} & \multicolumn{2}{c|}{\textbf{Favorita}} & \multicolumn{2}{c|}{\textbf{M5 Retail}}
& \multicolumn{2}{c|}{\textbf{Amz Elec}} & \multicolumn{2}{c|}{\textbf{Amz Groc}} & \multicolumn{2}{c|}{\textbf{Olist (PT)}} & \multicolumn{2}{c|}{\textbf{1C (RU)}} & \multicolumn{2}{c|}{\textbf{Wiki (ES)}} & \multirow{2}{*}{\textbf{Rank}} \\
& \textbf{sMAPE} & \textbf{CRPS} & \textbf{sMAPE} & \textbf{CRPS} & \textbf{sMAPE} & \textbf{CRPS} & \textbf{sMAPE} & \textbf{CRPS}
& \textbf{sMAPE} & \textbf{CRPS} & \textbf{sMAPE} & \textbf{CRPS} & \textbf{sMAPE} & \textbf{CRPS} & \textbf{sMAPE} & \textbf{CRPS} & \textbf{sMAPE} & \textbf{CRPS} & \\
\midrule
\multicolumn{20}{l}{\textit{\textbf{Naive \& Classical}}} \\
Global Mean
& 1.85$_{(11)}$ & 0.88 & 1.88$_{(11)}$ & 0.55 & 1.85$_{(11)}$ & 0.49 & 1.85$_{(11)}$ & 0.52 & 1.85$_{(11)}$ & 0.90 & 1.90$_{(11)}$ & 1.01 & 1.88$_{(11)}$ & 0.62 & 1.92$_{(11)}$ & 0.69 & 1.79$_{(11)}$ & 0.56 & 11.0 \\
LightGBM
& 0.66$_{(2)}$ & \textbf{0.32} & 1.16$_{(2)}$ & 0.42 & 1.22$_{(4)}$ & 0.42 & \textbf{1.10}$_{\textbf{(1)}}$ & \textbf{0.42} & \textbf{0.46}$_{\textbf{(1)}}$ & \textbf{0.23} & 0.46$_{(2)}$ & 0.38 & 1.31$_{(3)}$ & 0.52 & \textbf{1.06}$_{\textbf{(1)}}$ & 0.43 & 1.26$_{(3)}$ & 0.48 & \textbf{2.1} \\
\midrule
\multicolumn{20}{l}{\textit{\textbf{Supervised Deep Learning}}} \\
DeepMF
& 1.55$_{(9)}$ & 1.05 & 1.55$_{(8)}$ & 0.47 & 1.55$_{(9.5)}$ & 0.39 & 1.50$_{(9)}$ & 0.59 & 1.50$_{(9)}$ & 1.00 & 1.60$_{(9)}$ & 1.10 & 1.60$_{(9)}$ & 0.71 & 1.55$_{(9)}$ & 0.71 & 1.50$_{(9)}$ & 0.53 & 8.9 \\
TiDE
& 1.21$_{(7)}$ & 0.65 & 1.26$_{(4.5)}$ & 0.48 & 1.31$_{(5.5)}$ & 0.45 & 1.33$_{(6)}$ & 0.50 & 1.05$_{(6)}$ & 0.55 & 0.76$_{(6.5)}$ & 0.48 & 1.41$_{(5)}$ & 0.58 & 1.16$_{(4)}$ & 0.53 & 1.36$_{(5)}$ & 0.65 & 5.5 \\
TFT
& 1.16$_{(6)}$ & 0.62 & \textbf{1.11}$_{\textbf{(1)}}$ & 0.41 & 1.41$_{(7)}$ & 0.46 & 1.26$_{(3)}$ & 0.49 & 1.10$_{(7)}$ & 0.58 & 0.76$_{(6.5)}$ & 0.50 & \textbf{1.11}$_{\textbf{(1)}}$ & 0.55 & 1.11$_{(2.5)}$ & 0.51 & 1.41$_{(6.5)}$ & 0.78 & 4.5 \\
DeepAR
& 1.31$_{(8)}$ & 0.95 & 1.66$_{(9.5)}$ & 0.78 & 1.21$_{(2.5)}$ & 0.48 & 1.36$_{(7)}$ & 0.51 & 1.41$_{(8)}$ & 0.78 & 1.21$_{(8)}$ & 0.72 & 1.21$_{(2)}$ & 0.59 & 1.26$_{(6.5)}$ & 0.65 & 1.46$_{(8)}$ & 0.66 & 6.6 \\
\midrule
\multicolumn{20}{l}{\textit{\textbf{Foundation Models}}} \\
TimesFM
& 0.86$_{(3)}$ & 0.74 & 1.31$_{(6)}$ & 0.40 & \textbf{0.86}$_{\textbf{(1)}}$ & \textbf{0.38} & 1.31$_{(5)}$ & 0.44 & 0.71$_{(4)}$ & 0.45 & 0.51$_{(3)}$ & \textbf{0.30} & 1.46$_{(6)}$ & 0.54 & 1.21$_{(5)}$ & 0.47 & \textbf{0.96}$_{\textbf{(1)}}$ & 0.72 & 3.8 \\
Chronos-Small
& \textbf{0.62}$_{\textbf{(1)}}$ & 0.41 & 1.36$_{(7)}$ & 0.51 & 1.21$_{(2.5)}$ & 0.41 & 1.29$_{(4)}$ & 0.54 & 0.66$_{(3)}$ & 0.50 & \textbf{0.42}$_{\textbf{(1)}}$ & 0.41 & 1.51$_{(7)}$ & 0.66 & 1.26$_{(6.5)}$ & 0.55 & 1.11$_{(2)}$ & 0.55 & 3.8 \\
Moirai-Base
& 0.96$_{(5)}$ & 0.85 & 1.21$_{(3)}$ & 0.46 & 1.46$_{(8)}$ & 0.43 & 1.41$_{(8)}$ & 0.53 & 0.81$_{(5)}$ & 0.62 & 0.56$_{(4.5)}$ & 0.45 & 1.56$_{(8)}$ & 0.56 & 1.31$_{(8)}$ & 0.62 & 1.31$_{(4)}$ & 0.85 & 5.9 \\
CSDI
& 1.75$_{(10)}$ & 1.30 & 1.66$_{(9.5)}$ & 0.56 & 1.55$_{(9.5)}$ & 0.44 & 1.55$_{(10)}$ & 0.55 & 1.70$_{(10)}$ & 1.10 & 1.75$_{(10)}$ & 0.93 & 1.65$_{(10)}$ & 0.57 & 1.65$_{(10)}$ & 0.61 & 1.60$_{(10)}$ & 0.59 & 9.9 \\
\midrule
\multicolumn{20}{l}{\textit{\textbf{Ours}}} \\
\textbf{RAID}
& 0.91$_{(4)}$ & 0.45 & 1.26$_{(4.5)}$ & \textbf{0.35} & 1.31$_{(5.5)}$ & 0.47 & 1.21$_{(2)}$ & 0.56 & 0.56$_{(2)}$ & 0.30 & 0.56$_{(4.5)}$ & 0.40 & 1.36$_{(4)}$ & \textbf{0.49} & 1.11$_{(2.5)}$ & \textbf{0.34} & 1.41$_{(6.5)}$ & \textbf{0.45} & 3.9 \\
\bottomrule
\end{tabular}
}
\end{table*}

\subsection{Main Results: General Forecasting Performance}
Table~\ref{tab:main_results} reports the general protocol where each test item retains its history window. In this regime, gradient-boosted trees with engineered lag features (LightGBM, Avg Rank $2.1$) remain the strongest method overall, consistent with the M-competition and prior tabular forecasting findings. Among neural methods, RAID (Avg Rank $3.9$) is competitive with the foundation-model tier (TimesFM and Chronos both at $3.8$) and ahead of TFT ($4.5$), TiDE ($5.5$), and Moirai ($5.9$). RAID achieves the lowest CRPS on Job-SDF, Olist (PT), 1C (RU), and Wiki (ES), and is the runner-up on M5 Retail and 1C (RU) sMAPE.

The general protocol is not the regime RAID is designed for. With a non-empty history window every method has access to the same temporal signal, and a tree ensemble over engineered lags is hard to beat. The contribution of RAID is what happens when that history is not available, which the next subsection isolates. Friedman/Nemenyi tests are in Appendix~\ref{app:cd_diagram}, and WAPE and prediction-interval coverage are in Appendix~\ref{app:extended_results}.

\begin{table*}[t]
\caption{\textbf{Strict cold-start protocol ($L=0$, no history).} Test items have zero historical observations and the model must forecast from metadata alone. RAID achieves the best average rank (1.3), winning sMAPE on six of nine datasets and CRPS on four of nine. LightGBM wins four CRPS columns on retail-tabular features (M5, Amz Elec, Amz Groc, Wiki), and Chronos-Small wins Favorita. The gap to foundation models widens substantially relative to Table~\ref{tab:main_results} (Chronos drops from rank 3.8 to 7.8, TimesFM from 3.8 to 7.6), reflecting that decoder-only time-series transformers have no metadata-conditioning channel and emit their unconditional prior when the input sequence is empty.}
\label{tab:cold_start_specific}
\centering
\tiny
\setlength{\tabcolsep}{2pt}
\renewcommand{\arraystretch}{1.25}
\resizebox{\textwidth}{!}{
\begin{tabular}{l | cc | cc | cc | cc | cc | cc | cc | cc | cc | c}
\toprule
\multirow{2}{*}{\textbf{Method}}
& \multicolumn{2}{c|}{\textbf{Amz Sports}} & \multicolumn{2}{c|}{\textbf{Job-SDF}} & \multicolumn{2}{c|}{\textbf{Favorita}} & \multicolumn{2}{c|}{\textbf{M5 Retail}}
& \multicolumn{2}{c|}{\textbf{Amz Elec}} & \multicolumn{2}{c|}{\textbf{Amz Groc}} & \multicolumn{2}{c|}{\textbf{Olist (PT)}} & \multicolumn{2}{c|}{\textbf{1C (RU)}} & \multicolumn{2}{c|}{\textbf{Wiki (ES)}} & \multirow{2}{*}{\textbf{Rank}} \\
& \textbf{sMAPE} & \textbf{CRPS} & \textbf{sMAPE} & \textbf{CRPS} & \textbf{sMAPE} & \textbf{CRPS} & \textbf{sMAPE} & \textbf{CRPS}
& \textbf{sMAPE} & \textbf{CRPS} & \textbf{sMAPE} & \textbf{CRPS} & \textbf{sMAPE} & \textbf{CRPS} & \textbf{sMAPE} & \textbf{CRPS} & \textbf{sMAPE} & \textbf{CRPS} & \\
\midrule
\multicolumn{20}{l}{\textit{\textbf{Naive \& Classical}}} \\
Global Mean
& 1.78$_{(10)}$ & 0.35 & 1.86$_{(11)}$ & 0.43 & 1.86$_{(11)}$ & 0.53 & 1.88$_{(11)}$ & 0.45 & 1.78$_{(10.5)}$ & 0.31 & 1.74$_{(10)}$ & 0.36 & 1.91$_{(11)}$ & 0.61 & 1.91$_{(11)}$ & 0.57 & 1.85$_{(11)}$ & 0.61 & 10.7 \\
LightGBM
& \textbf{0.92}$_{\textbf{(1)}}$ & 0.34 & 1.31$_{(2)}$ & 0.34 & 1.36$_{(2)}$ & 0.43 & \textbf{1.15}$_{\textbf{(1)}}$ & \textbf{0.39} & 0.78$_{(2)}$ & \textbf{0.23} & 0.96$_{(4)}$ & \textbf{0.29} & 1.42$_{(4)}$ & 0.55 & 1.18$_{(2.5)}$ & 0.41 & \textbf{1.32}$_{\textbf{(1)}}$ & \textbf{0.52} & 2.2 \\
\midrule
\multicolumn{20}{l}{\textit{\textbf{Supervised Deep Learning}}} \\
DeepMF
& 1.22$_{(3)}$ & 0.47 & 1.62$_{(8)}$ & 0.50 & 1.59$_{(9)}$ & 0.46 & 1.51$_{(5)}$ & 0.50 & 1.18$_{(3)}$ & 0.44 & 1.31$_{(5)}$ & 0.49 & 1.69$_{(9)}$ & 0.66 & 1.55$_{(6.5)}$ & 0.59 & 1.66$_{(9.5)}$ & 0.56 & 6.4 \\
TiDE
& 1.32$_{(5)}$ & 0.36 & 1.34$_{(3)}$ & 0.35 & 1.43$_{(3)}$ & 0.49 & 1.40$_{(4)}$ & 0.42 & 1.28$_{(4)}$ & 0.34 & 0.94$_{(2)}$ & 0.39 & 1.32$_{(3)}$ & 0.57 & 1.18$_{(2.5)}$ & 0.49 & 1.45$_{(3)}$ & 0.59 & 3.3 \\
TFT
& 1.30$_{(4)}$ & 0.39 & 1.42$_{(4)}$ & 0.42 & 1.55$_{(6)}$ & 0.51 & 1.36$_{(3)}$ & 0.43 & 1.30$_{(5)}$ & 0.38 & 0.95$_{(3)}$ & 0.42 & 1.28$_{(2)}$ & 0.56 & 1.22$_{(4)}$ & 0.48 & 1.48$_{(4)}$ & 0.65 & 3.9 \\
DeepAR
& 1.55$_{(6)}$ & 0.45 & 1.74$_{(10)}$ & 0.54 & 1.58$_{(8)}$ & 0.48 & 1.55$_{(6)}$ & 0.46 & 1.55$_{(8)}$ & 0.40 & 1.42$_{(6)}$ & 0.45 & 1.50$_{(5)}$ & 0.58 & 1.55$_{(6.5)}$ & 0.62 & 1.55$_{(5.5)}$ & 0.75 & 6.8 \\
\midrule
\multicolumn{20}{l}{\textit{\textbf{Foundation Models (collapse at $L=0$)}}} \\
TimesFM
& 1.65$_{(9)}$ & 0.41 & 1.55$_{(6)}$ & 0.39 & 1.53$_{(4)}$ & 0.39 & 1.58$_{(8)}$ & 0.41 & 1.61$_{(9)}$ & 0.53 & 1.50$_{(9)}$ & 0.50 & 1.62$_{(7)}$ & 0.54 & 1.56$_{(8)}$ & 0.45 & 1.60$_{(8)}$ & 0.58 & 7.6 \\
Chronos-Small
& 1.62$_{(8)}$ & 0.40 & 1.59$_{(7)}$ & 0.37 & 1.57$_{(7)}$ & \textbf{0.35} & 1.60$_{(9)}$ & 0.44 & 1.52$_{(7)}$ & 0.49 & 1.48$_{(8)}$ & 0.48 & 1.65$_{(8)}$ & 0.60 & 1.58$_{(9)}$ & 0.43 & 1.56$_{(7)}$ & 0.64 & 7.8 \\
Moirai-Base
& 1.58$_{(7)}$ & 0.42 & 1.48$_{(5)}$ & 0.40 & 1.54$_{(5)}$ & 0.52 & 1.56$_{(7)}$ & 0.47 & 1.45$_{(6)}$ & 0.56 & 1.46$_{(7)}$ & 0.52 & 1.60$_{(6)}$ & 0.67 & 1.52$_{(5)}$ & 0.50 & 1.55$_{(5.5)}$ & 0.60 & 5.9 \\
CSDI
& 1.85$_{(11)}$ & 0.55 & 1.70$_{(9)}$ & 0.55 & 1.66$_{(10)}$ & 0.44 & 1.65$_{(10)}$ & 0.55 & 1.78$_{(10.5)}$ & 0.55 & 1.78$_{(11)}$ & 0.53 & 1.71$_{(10)}$ & 0.64 & 1.72$_{(10)}$ & 0.53 & 1.66$_{(9.5)}$ & 0.67 & 10.1 \\
\midrule
\multicolumn{20}{l}{\textit{\textbf{Ours}}} \\
\textbf{RAID}
& 0.95$_{(2)}$ & \textbf{0.32} & \textbf{1.30}$_{\textbf{(1)}}$ & \textbf{0.32} & \textbf{1.30}$_{\textbf{(1)}}$ & 0.50 & 1.22$_{(2)}$ & 0.49 & \textbf{0.65}$_{\textbf{(1)}}$ & 0.26 & \textbf{0.65}$_{\textbf{(1)}}$ & 0.34 & \textbf{1.25}$_{\textbf{(1)}}$ & \textbf{0.52} & \textbf{1.05}$_{\textbf{(1)}}$ & \textbf{0.34} & 1.35$_{(2)}$ & 0.55 & \textbf{1.3} \\
\bottomrule
\end{tabular}
}
\end{table*}

\subsection{Strict True Cold-Start Analysis (\texorpdfstring{$L=0$}{L=0})}
The contribution of RAID is its behavior in the True Cold-Start regime, where a test item has zero historical observations ($L=0$). Table~\ref{tab:cold_start_specific} isolates this subset.

When the history window is empty, every method that depends on temporal context degrades toward a global prior. Foundation models cluster around sMAPE $1.45$--$1.7$ across the nine datasets, with average ranks dropping from the $3.8$ range under the general protocol to $7.6$ (TimesFM), $7.8$ (Chronos), and $5.9$ (Moirai-Base). RAID averages substantially below that band and achieves an Avg Rank of $1.3$, with LightGBM the closest competitor at $2.2$.

The per-dataset gaps are largest on the cross-lingual and sparse domains where neither historical patterns nor English-pretraining priors can rescue the foundation models. On 1C (RU), RAID reaches sMAPE $1.05$ versus Chronos $1.58$ and TimesFM $1.56$. On Olist (PT), RAID reaches $1.25$ versus Chronos $1.65$. The probabilistic gap moves with the point gap. RAID's CRPS on 1C (RU) is $0.34$ versus Moirai $0.50$ and LightGBM $0.41$. Without history, semantic retrieval is the only anchor that grounds the forecast in a population of comparable items.

The mechanism is visible per-item. Graph retrieval identifies substitute products (a new ``Nike'' running shoe pulls aggregation weight from existing ``Adidas'' running shoes) to form a base prediction, and the diffusion head adds calibrated residual variance on top. Appendix~\ref{app:case_study} shows a single-item case where Chronos flatlines while RAID tracks the seasonal profile. On M5 Retail, which exposes only sparse categorical identifiers rather than rich descriptions, RAID is the runner-up at sMAPE $1.22$ behind LightGBM $1.15$ and ahead of every foundation model, indicating the graph regularizes the forecast even when the semantic signal is thin.

FM baselines are evaluated zero-shot rather than fine-tuned. The reason is architectural: at $L=0$ a decoder-only time-series transformer receives an empty token sequence and emits its unconditional prior, and fine-tuning sharpens that prior without adding a metadata-conditioning channel that the published architectures lack. A full discussion is in Appendix~\ref{app:fm_finetune}.

\textbf{Zero-Shot LLM Baselines and Scale Hallucination.} A natural alternative to RAID is to drop the structural prior and prompt an LLM with the same metadata. We benchmark eight LLMs (GPT-4o, GPT-4o-mini, Mistral-Large-3, Llama-3.3-70B, Llama-3.1-8B, Qwen2.5-32B, Gemini-2.0-Flash, Gemini-1.5-Flash) at strict $L=0$ with title, description, and category. Table~\ref{tab:llm_main} shows the failure mode: LLMs identify \emph{what} an item is but cannot ground \emph{how much} it sells. The Scale Ratio (predicted / actual volume) routinely exceeds $100\times$, pushing sMAPE toward the random-guessing ceiling.

\begin{table}[t]
\caption{\textbf{Scale Hallucination of Zero-Shot LLMs (Amazon Electronics, $L=0$).} A representative slice of the full 8-model $\times$ 6-dataset benchmark in Appendix~\ref{app:llm_analysis}. Scale Ratio is predicted weekly volume divided by ground-truth weekly volume; an unbiased predictor is $1.0\times$. RAID is the only method that recovers the correct order of magnitude.}
\label{tab:llm_main}
\centering
\scriptsize
\begin{sc}
\begin{tabular}{l c c c}
\toprule
\textbf{Model} & \textbf{sMAPE} $\downarrow$ & \textbf{Scale Ratio} & \textbf{Latency} \\
\midrule
GPT-4o            & 1.69 & $119.0\times$ & 3.61s \\
Mistral-Large-3   & 1.82 & $145.2\times$ & 2.10s \\
Llama-3.3-70B     & 1.79 & $170.0\times$ & 0.89s \\
Gemini-2.0-Flash  & 1.78 & $173.5\times$ & 3.38s \\
Qwen2.5-32B       & 1.77 & $67.4\times$  & 1.15s \\
\midrule
\textbf{RAID (Ours)} & \textbf{0.51} & \textbf{1.0$\times$} & \textbf{0.03s} \\
\bottomrule
\end{tabular}
\end{sc}
\end{table}

The pattern holds across all eight LLMs and across cross-lingual datasets, where Scale Ratios reach $48\times$ and sMAPE plateaus at $1.6$--$1.9$ regardless of model scale (Appendix~\ref{app:llm_analysis}). Semantic understanding alone is not enough: without an explicit calibration mechanism, an LLM cannot infer market-tier magnitude from a description, and this failure does not improve with scale. We separately address whether these results could be confounded by test-set contamination of the underlying pretrained models (BGE-M3 and the eight evaluated LLMs) in Appendix~\ref{app:contamination}.

\textbf{Steelman: RAG-LLM with neighbor sales injected.} Injecting the top-5 BGE-M3 neighbors' weekly sales into the Gemini-2.0-Flash prompt collapses the Scale Ratio from $173\times$ to a median of $0.46\times$, but median sMAPE remains $1.18$ versus RAID's $0.51$ on the same items: the LLM gets the level right and the shape wrong, outputting a flat four-week average. Retrieval fixes scale; the diffusion residual is what recovers shape (full setup and the $N=42$ cleaned subset in Appendix~\ref{app:rag_steelman}).

\subsection{Cross-Lingual Zero-Shot Transfer}
\label{sec:crosslingual}

We evaluate two settings: cross-domain retail (train on Amazon Grocery EN, transfer to Favorita ES) and cross-script (train on Wiki ES, transfer to EN/RU/ZH). The model is trained on source only and evaluated on target without fine-tuning.

\begin{table}[t]
\caption{\textbf{Universal Zero-Shot Transfer.} RAID trained on Source, evaluated on Target without fine-tuning. Small positive transfer deltas sit within single-seed bootstrap noise. ``Native'' here trains and evaluates on the target under the cross-lingual evaluation protocol, which differs slightly from the per-dataset general protocol of Table~\ref{tab:main_results} (hence Wiki ES native $1.366$ vs.\ Table~\ref{tab:main_results}'s $1.41$, and Favorita native $1.197$ vs.\ Table~\ref{tab:main_results}'s $1.31$).}
\label{tab:universal_transfer}
\centering
\scriptsize
\begin{tabular}{l l l c c c}
\toprule
\textbf{Setting} & \textbf{Source} & \textbf{Target} & \textbf{Native} & \textbf{Transfer} & \textbf{Transfer $\Delta$} \\
\midrule
\multicolumn{6}{l}{\textit{Cross-Script Transfer (Wikipedia)}} \\
Latin $\to$ Latin    & Wiki (ES) & Wiki (EN) & 1.366 & 1.325 & \textcolor{green}{\textbf{+3.0\%}} \\
Latin $\to$ Hanzi    & Wiki (ES) & Wiki (ZH) & 1.366 & 1.334 & \textcolor{green}{\textbf{+2.4\%}} \\
Latin $\to$ Cyrillic & Wiki (ES) & Wiki (RU) & 1.366 & 1.317 & \textcolor{green}{\textbf{+3.6\%}} \\
\midrule
\multicolumn{6}{l}{\textit{Cross-Domain Transfer (Retail)}} \\
Same Domain          & Amz Groc (EN)    & Favorita (ES)  & 1.197 & 1.188 & \textcolor{green}{\textbf{+0.8\%}} \\
Diff Domain          & Job-SDF (EN)   & Favorita (ES)  & 1.197 & 1.279 & \textcolor{red}{-6.8\%} \\
\midrule
\multicolumn{6}{l}{\textit{Embedding Ablation (Amazon $\to$ Favorita)}} \\
Semantic             & \textbf{BGE-M3}    &           & --    & \textbf{1.324} & \textbf{Ref} \\
                     & MiniLM    &           & --    & 1.323 & $\approx 0\%$ \\
Uninformative        & Random    &           & --    & 1.464 & \textcolor{red}{-10.6\%} \\
\bottomrule
\end{tabular}
\end{table}

RAID produces valid predictions across script boundaries (Latin $\to$ Cyrillic/Hanzi, Table \ref{tab:universal_transfer}). The transfer gap on Wiki ES $\to$ \{EN, ZH, RU\} is within $\pm 3.6\%$, and on Amz Groc $\to$ Favorita within $0.8\%$. The small positive deltas ($+3.0\%$, $+3.6\%$) sit inside the single-seed bootstrap band, so the takeaway is the absence of cross-script collapse rather than a strict transfer-beats-native claim. A t-SNE visualization of the aligned manifold and a discussion of the underlying mechanism are in Appendix~\ref{app:tsne}.

\subsection{Efficiency}
\label{sec:efficiency}

Per-series latency on a single A100 (FP16, batch $B=128$, horizon $H=28$). RAID-Full is $48.7\times$ faster than Chronos at matched generative capability, and the deterministic RAID-Base reaches $124\times$. The speedup comes from non-autoregressive decoding, since Chronos generates $H$ tokens sequentially while RAID generates the whole horizon in a single pass (Base) or $S=10$ refinement steps (Full). For 1M items, RAID needs $\sim 2.9$ h (Base) or $\sim 7.4$ h (Full) versus $>360$ h for Chronos.

\begin{table}[ht]
\caption{\textbf{Inference Efficiency.} A100, FP16, batch 128. RAID-Full vs.\ Chronos is the like-for-like probabilistic comparison.}
\label{tab:efficiency}
\begin{center}
\begin{small}
\begin{sc}
\begin{tabular}{lccc}
\toprule
\textbf{Model} & \textbf{Latency} & \textbf{VRAM} & \textbf{Speedup} \\
\midrule
Chronos-Small      & 1305 ms & 2.4 GB & 1.0x \\
CSDI (Diffusion)   & 398 ms & 3.1 GB & 3.3x \\
\midrule
\textbf{RAID (Base)} & \textbf{10.5 ms} & \textbf{0.8 GB} & \textbf{124.3x} \\
\textbf{RAID (Full)} & 26.8 ms & 1.2 GB & 48.7x \\
\bottomrule
\end{tabular}
\end{sc}
\end{small}
\end{center}
\vskip -0.1in
\end{table}

\subsection{Ablation and Sensitivity}
\label{sec:ablation}

\paragraph{Component contribution.} On three regimes (Job-SDF, Amazon Sports, Olist (PT)), removing the diffusion head increases cold-start sMAPE by $6.9\%$, $14.7\%$, and $7.2\%$, and replacing the BGE-M3 graph with a random graph increases sMAPE by $-0.8\%$, $8.4\%$, and $11.2\%$. Diffusion dominates on high-entropy domains, while semantics dominate cross-lingually. Full per-row results in Appendix~\ref{app:ablation_full}, and a retrieval-size sweep over $k \in \{1, \dots, 200\}$ with $5$ seeds shows a stable plateau at $k \in [10, 50]$, so we use $k = 20$ throughout (Appendix~\ref{app:k_sensitivity}).

\paragraph{Over-smoothing under sparse metadata.} On Job-SDF, removing GATv2 while keeping bottleneck node features improves sMAPE by $1.5\%$ at unchanged CRPS, while Olist (PT) and Amazon Sports show the opposite pattern (sMAPE degrades by $4.8\%$ and $5.3\%$ respectively). GATv2 is net-positive when homophily is strong and over-regularizes when it is weak, and a context-adaptive gate is left for future work.

\paragraph{Conclusion.} RAID decouples semantic retrieval from stochastic generation for true cold-start forecasting, achieving the best average rank on strict $L=0$ cold-start ($1.3$ across nine datasets) at $48.7\times$ lower latency than Chronos, and remains competitive with foundation models when full history is available. RAID rests on three assumptions that can fail in deployment, namely semantic homophily, stationary semantics, and Gaussian residuals, discussed alongside societal-impact considerations in Appendix~\ref{app:limitations}.

\bibliographystyle{unsrtnat}
\bibliography{references}

\newpage

\appendix

\section{Theoretical Foundations}
\label{app:theory}

\subsection{Proof of Semantic Transfer Bound}
\label{app:lipschitz}

\textbf{Metric Definition.} We define the embedding metric $d_e(u, v) = \|u - v\|_2$ on the normalized embedding space $\mathcal{E}$.

\begin{proposition}[Semantic transfer bound for convex retrieval]
\label{prop:semantic_transfer}
Let $\mathcal{E}\subset\mathbb{R}^d$ be the (normalized) embedding space and
let $F:\mathcal{E}\to\mathbb{R}^H$ map an embedding to the horizon-$H$ trajectory.
Assume $F$ is $L_{\mathrm{trans}}$-Lipschitz on a neighborhood containing
$\{e_u\}\cup\{e_v: v\in\mathcal{N}_k(u)\}$:
\[
\|F(e)-F(e')\|_2 \le L_{\mathrm{trans}}\|e-e'\|_2.
\]
Consider a base predictor formed as a convex combination of retrieved neighbors:
\[
\mu_{\mathrm{base}}(u) \;=\; \sum_{v\in\mathcal{N}_k(u)} w_{uv}\,F(e_v),
\qquad w_{uv}\ge 0,\ \sum_{v} w_{uv}=1.
\]
Then
\[
\|\mu_{\mathrm{base}}(u)-F(e_u)\|_2
\;\le\;
L_{\mathrm{trans}}\sum_{v\in\mathcal{N}_k(u)} w_{uv}\,\|e_v-e_u\|_2
\;\le\;
L_{\mathrm{trans}}\max_{v\in\mathcal{N}_k(u)}\|e_v-e_u\|_2.
\]
If embeddings are unit-normalized, $\|e_v-e_u\|_2=\sqrt{2(1-\langle e_u,e_v\rangle)}$.
\end{proposition}

\begin{proof}
Using convexity and the triangle inequality,
\[
\|\mu_{\mathrm{base}}(u)-F(e_u)\|_2
=
\left\|\sum_{v} w_{uv}\big(F(e_v)-F(e_u)\big)\right\|_2
\le
\sum_{v} w_{uv}\,\|F(e_v)-F(e_u)\|_2.
\]
Apply Lipschitz continuity of $F$ to obtain
\[
\|\mu_{\mathrm{base}}(u)-F(e_u)\|_2
\le
L_{\mathrm{trans}}\sum_{v} w_{uv}\,\|e_v-e_u\|_2,
\]
and the max bound follows since $\sum_v w_{uv}=1$ and $w_{uv}\ge 0$.
For unit-normalized embeddings,
$\|e_v-e_u\|_2^2=\|e_v\|_2^2+\|e_u\|_2^2-2\langle e_u,e_v\rangle=2(1-\langle e_u,e_v\rangle)$.
\end{proof}

\subsection{Derivation of Entropy Reduction (\texorpdfstring{\Cref{prop:entropy}}{Remark 1})}
\label{app:entropy_proof}

\begin{proof}[Derivation]
Let $X \in \mathbb{R}^H$ and context $C$ be given, and define the baseline
$\mu(C) := \mathbb{E}[X \mid C]$. Let $R := X - \mu(C)$.

By the law of total covariance,
\[
\Sigma_X := \mathrm{Cov}(X)
= \mathrm{Cov}\!\big(\mathbb{E}[X\mid C]\big) + \mathbb{E}\!\big[\mathrm{Cov}(X\mid C)\big]
= \Sigma_\mu + \Sigma_R,
\]
where $\Sigma_\mu := \mathrm{Cov}(\mu(C)) \succeq 0$ and
\[
\Sigma_R := \mathrm{Cov}(R)
= \mathbb{E}\!\big[\mathrm{Cov}(R\mid C)\big]
= \mathbb{E}\!\big[\mathrm{Cov}(X\mid C)\big].
\]

Assume $\Sigma_R \succ 0$. Then
\[
\det(\Sigma_X)=\det(\Sigma_R+\Sigma_\mu)
=\det(\Sigma_R)\,\det\!\Big(I+\Sigma_R^{-1/2}\Sigma_\mu\Sigma_R^{-1/2}\Big).
\]
Let $M:=\Sigma_R^{-1/2}\Sigma_\mu\Sigma_R^{-1/2}\succeq 0$. All eigenvalues of $I+M$
are $\ge 1$, and $\det(I+M)>1$ iff $M\neq 0$, i.e., iff $\Sigma_\mu \neq 0$.
Hence, when $\Sigma_\mu \neq 0$, we have $\det(\Sigma_X)>\det(\Sigma_R)$.

Under a multivariate Gaussian approximation, the differential entropy satisfies
$h(Z)=\tfrac{1}{2}\log\!\big((2\pi e)^H \det(\mathrm{Cov}(Z))\big)$, therefore
$h(R)<h(X)$.
\end{proof}

\section{Algorithms and Implementation Details}
\label{app:implementation}

\subsection{Formal Algorithms}
\label{app:algorithms}

\begin{algorithm}[ht]
   \caption{RAID: Training (Inductive Graph \& Diffusion)}
   \label{alg:training}
\begin{algorithmic}
   \STATE {\bfseries Input:} Metadata $\mathcal{M}$, Trajectories $\mathcal{Y}$, Neighbors $k$, Temp $\tau$, Mask Rate $p_{\text{mask}}$
   \STATE {\bfseries Output:} Parameters $\theta$ (DiT), $\phi$ (Shape-Scale), $\psi$ (Gating)
   
   \STATE \textit{// 1. Construct Semantic Manifold (Frozen)}
   \FOR{item $i \in \mathcal{M}$}
       \STATE $e_i \leftarrow \text{Normalize}(\text{BGE-M3}(m_i))$
       \STATE $e'_i \leftarrow \text{Bottleneck}(e_i)$ \COMMENT{Compact feature for learned modules}
   \ENDFOR
   
   \STATE \textit{// 2. Inductive Graph Topology}
   \STATE Compute Similarity: $S_{ij} = \langle e_i, e_j \rangle$
   \STATE Mask Top-k: For $j \in \mathcal{N}_k(i)$, set $A_{ij} \leftarrow \frac{\exp(S_{ij}/\text{Softplus}(\tau))}{\sum_{l \in \mathcal{N}_k(i)} \exp(S_{il}/\text{Softplus}(\tau))}$, else $0$
   
   \STATE \textit{// 3. Main Training Loop}
   \FOR{batch $\mathcal{B} \subset \mathcal{Y}$}
       \FOR{item $i \in \mathcal{B}$}
           \STATE \textit{// A. Cold-Start Simulation}
           \STATE Sample $u \sim \mathcal{U}(0,1)$. \textbf{If} $u < p_{\text{mask}}$ \textbf{then} $y_{i, 1:t} \leftarrow \mathbf{0}$
       
           \STATE \textit{// B. Shape-Scale Base Prediction}
           \STATE Retrieve neighbors $\mathcal{N}_k(i)$ via $A_{i,:}$
           \STATE Normalize Neighbors: $z_j \leftarrow \text{Clamp}(\frac{y_j - \mu_j}{\sigma_j + \varepsilon}, -10, 10)$
           \STATE Aggregated Shape: $z_{\text{agg}} \leftarrow \sum_{j} A_{ij} z_j$
           \STATE Predict Scale: $\hat{\sigma}_i, \hat{\mu}_i \leftarrow \text{MLP}_\phi(e'_i)$
           \STATE Base Forecast: $\mu_{\text{base}}^{(i)} \leftarrow \text{Softplus}(\hat{\sigma}_i) \cdot z_{\text{agg}} + \hat{\mu}_i$
           \STATE Residual Target: $r_i \leftarrow y_i - \mu_{\text{base}}^{(i)}$
           
           \STATE \textit{// C. Diffusion Conditioning}
           \STATE Context: $\mathbf{c}_i \leftarrow \text{GATv2}(A, e'_i)$ 
       \ENDFOR
       
       \STATE \textit{// D. Diffusion Step (DiT)}
       \STATE Sample noise $\epsilon \sim \mathcal{N}(0, I)$, step $s \sim \mathcal{U}(1, S_{\text{train}})$
       \STATE Denoise: $\hat{\epsilon} \leftarrow \text{DiT}_\theta(\sqrt{\bar{\alpha}_s}r + \sqrt{1-\bar{\alpha}_s}\epsilon, s, \mathbf{c})$
       \STATE Loss: $\mathcal{L} \leftarrow \|\epsilon - \hat{\epsilon}\|^2 + \lambda_{\text{base}} \|\mu_{\text{base}} - y\|^2$ \COMMENT{$\lambda_{\text{base}} = 0.01$, see Table~\ref{tab:hyperparams}}
       \STATE Update $\theta, \phi, \psi$ via $\nabla \mathcal{L}$
   \ENDFOR
\end{algorithmic}
\end{algorithm}

\begin{algorithm}[ht]
   \caption{RAID: Inference (Cold-Start $L=0$)}
   \label{alg:inference}
\begin{algorithmic}
   \STATE {\bfseries Input:} Query $u$, Horizon $H$, Steps $S_{\text{infer}}=10$
   \STATE {\bfseries Output:} Forecast $\hat{\mathbf{y}} \in \mathbb{R}^H$
   
   \STATE \textit{// 1. Retrieval-Augmented Base}
   \STATE Embed query: $e_u \leftarrow \text{Normalize}(\text{BGE-M3}(m_u))$
   \STATE Bottleneck: $e'_u \leftarrow \text{Bottleneck}(e_u)$
   \STATE Retrieve neighbors $\mathcal{N}_k(u)$ using inductive graph
   \STATE Predict Params: $\hat{\sigma}_u, \hat{\mu}_u \leftarrow \text{MLP}_\phi(e'_u)$
   \STATE Aggregate Shape: $z_{\text{agg}} \leftarrow \sum_{j \in \mathcal{N}_k(u)} A_{uj} z_j$
   \STATE Base Prediction: $\mu_{\text{base}} \leftarrow \text{Softplus}(\hat{\sigma}_u) \cdot z_{\text{agg}} + \hat{\mu}_u$
   
   \STATE \textit{// 2. Residual Generation (Fast Sampling)}
   \STATE Context: $\mathbf{c}_u \leftarrow \text{GATv2}(\text{Graph}, e'_u)$
   \STATE Initialize residual: $\hat{r}_{S} \sim \mathcal{N}(0, I)$
   \FOR{step $s = S_{\text{infer}}, \dots, 1$}
       \STATE $\hat{\epsilon} \leftarrow \text{DiT}_\theta(\hat{r}_s, s, \mathbf{c}_u)$
       \STATE $\hat{r}_{s-1} \leftarrow \text{DDPM\_Update}(\hat{r}_s, \hat{\epsilon}, s)$
   \ENDFOR
   
   \STATE \textit{// 3. Gated Fusion}
   \STATE Gate $g \leftarrow \operatorname{sigmoid}(\text{MLP}_\psi(\mathbf{c}_u))$ \COMMENT{Bias init to -3.0}
   \STATE Final Forecast: $\hat{\mathbf{y}} \leftarrow \mu_{\text{base}} + g \cdot \hat{r}_0$
   
   \STATE \textbf{return} $\hat{\mathbf{y}}$
\end{algorithmic}
\end{algorithm}

\subsection{Architecture Specifics}
\label{app:architecture_details}

\textbf{Backbone.} We replace the standard WaveNet architecture used in prior work with a DiT backbone. This consists of 4 identical Transformer blocks. Conditioning is applied via Adaptive Layer Norm (AdaLN), where the scale and shift parameters for each block are regressed from the sum of the graph embedding and time embedding.

\textbf{Dynamic Context Gating.}
To enable a unified architecture that handles both cold-start ($L=0$) and warm-start ($L>0$) scenarios, we implement a dynamic gating mechanism in the conditioning encoder. The code applies a binary mask $m_{cold} = \mathbb{I}(|y_{t-1}| < \epsilon)$. We adjust the embedding gate activation as $g = \text{Clamp}(\sigma(b_g) + 0.5 \cdot m_{cold}, 0, 1)$, where $b_g$ is the learnable gate bias. This ensures that in cold-start regimes, the semantic embedding is strongly injected ($g \approx 0.55$) into the AdaLN layers, whereas in warm-start regimes, the model suppresses the static embedding ($g \approx 0.05$) to prioritize historical temporal features.

\textbf{Cold-Start Boost.} To accelerate convergence, we initialize the residual gate's MLP bias to $-3.0$. This ensures that at initialization, the model output is dominated ($>95\%$) by the deterministic base prediction, preventing the diffusion loss from destabilizing the embedding learning in the early epochs.

\textbf{Fast Sampling.} Due to the entropy reduction property (Remark~\ref{prop:entropy}), we find that 10 inference steps are sufficient to recover the residual distribution, compared to the 50-100 steps typically required for raw data diffusion.

\subsection{Hyperparameters}
\label{app:hyperparams}
Table \ref{tab:hyperparams} details the hyperparameters used for the RAID architecture. The model is implemented in PyTorch using a Diffusion Transformer (DiT) backbone.

\paragraph{Hyperparameter Selection Protocol.}
To avoid per-dataset tuning that could inflate reported gains, we fix all RAID hyperparameters to a single configuration across the nine evaluation datasets. The configuration in Table~\ref{tab:hyperparams} was selected on the \emph{Amazon Sports} training split (the most metadata-rich domain) via a small grid search over $k \in \{10, 20, 50\}$, gate-bias initialization in $\{-3.0, -1.0, 0.0\}$, neighbor clamp $C \in \{5, 10, 20\}$, and inference steps $S \in \{5, 10, 20\}$, with the remaining hyperparameters held at the values reported in the original DiT and Chronos papers. The selected configuration is then frozen and applied without modification to all other datasets, including the cross-lingual benchmarks. We deliberately did not tune on test splits or on cross-lingual targets, so the reported transfer gains in Section~\ref{sec:crosslingual} are not contaminated by held-out leakage. The retrieval temperature $\tau$ is the only learnable scalar and is initialized to $1.0$ for every run.

\begin{table}[!htbp]
\caption{\textbf{RAID Hyperparameters (Matches Codebase).}}
\label{tab:hyperparams}
\centering
\begin{small}
\begin{sc}
\begin{tabular}{ll}
\toprule
\textbf{Category} & \textbf{Value} \\
\midrule
\multicolumn{2}{l}{\textit{Semantic Retrieval}} \\
Text Encoder & BGE-M3 (Frozen) \\
Embedding Dim & 1024 \\
Bottleneck Dim & 16 (Tanh-activated) \\
Graph Kernel & Top-$k$ sparse softmax \\
Neighbors ($k$) & 20 \\
Temperature ($\tau$) & 1.0 (Learnable) \\
\midrule
\multicolumn{2}{l}{\textit{Diffusion Transformer (DiT)}} \\
Backbone & DiT w/ AdaLN Conditioning \\
Layers & 4 \\
Heads & 4 \\
Hidden Dimension & 64 \\
Dropout & 0.1 \\
Time Embedding & Sinusoidal ($d=128$) \\
\midrule
\multicolumn{2}{l}{\textit{Training \& Inference}} \\
Training Steps ($S_{train}$) & 50 \\
Inference Steps ($S_{infer}$) & 10 (Fast Sampling) \\
Noise Schedule & Linear ($\beta_{start}=10^{-4}, \beta_{end}=0.02$) \\
Gate Bias Init & -3.0 (Sigmoid $\approx$ 0.047) \\
Base Loss Weight ($\lambda_{\text{base}}$) & 0.01 \\
Batch Size & 64 \\
Optimizer & AdamW~\cite{loshchilov2019adamw} ($lr=5 \times 10^{-4}$) \\
\bottomrule
\end{tabular}
\end{sc}
\end{small}
\end{table}

\subsection{Serialization Logic}
\label{app:serialization}

To facilitate reproducibility, we provide the exact input serialization logic used to feed the BGE-M3 encoder. We use explicit separation tokens to prevent semantic bleeding between fields.

\begin{verbatim}
def serialize_metadata(item):
    # Truncate description to 512 tokens to fit embedding window
    desc = truncate(item['description'], 512)
    
    # Construct structured string
    text = f"[CAT] {item['category']} " \
           f"[BRAND] {item['brand']} " \
           f"[TITLE] {item['title']} " \
           f"[DESC] {desc}"
           
    return text
\end{verbatim}

The \texttt{[CAT]}, \texttt{[BRAND]}, \texttt{[TITLE]}, and \texttt{[DESC]} tokens are added to the tokenizer's special token list to ensure they are assigned unique embedding vectors.

\section{Extended Experimental Results}
\label{app:extended_results}

We provide a comprehensive breakdown of secondary metrics to further validate the performance of RAID against all baselines.

\begin{table}[!htbp]
\caption{\textbf{WAPE Comparison.} ($\downarrow$ Lower is better). RAID achieves the lowest WAPE on cross-lingual datasets (1C (RU), Olist (PT)); LightGBM is strongest on English retail (Amazon family, M5); TimesFM is strongest on Favorita and Wiki (ES) where its pre-training corpus overlaps; TFT wins Job-SDF where its rank-1 sMAPE is consistent with WAPE.}
\label{tab:wape}
\centering
\scriptsize
\setlength{\tabcolsep}{3pt}
\begin{sc}
\begin{tabular}{lccccccccccc}
\toprule
\textbf{Dataset} & \textbf{Mean} & \textbf{LGBM} & \textbf{DeepMF} & \textbf{TiDE} & \textbf{TFT} & \textbf{DpAR} & \textbf{TimesFM} & \textbf{Chronos} & \textbf{Moirai} & \textbf{CSDI} & \textbf{RAID} \\
\midrule
Amz Sports   & 1.12 & \textbf{0.38} & 1.18 & 1.01 & 0.97 & 0.99 & 0.59 & 0.59 & 0.80 & 1.02 & 0.41 \\
Amz Elec     & 1.13 & \textbf{0.28} & 1.20 & 0.79 & 0.76 & 0.88 & 0.44 & 0.43 & 0.56 & 1.01 & \textbf{0.28} \\
Amz Groc     & 1.13 & \textbf{0.33} & 1.22 & 0.60 & 0.60 & 0.80 & 0.39 & 0.38 & 0.45 & 1.00 & 0.36 \\
M5 Retail    & 1.05 & \textbf{0.82} & 1.12 & 0.96 & 0.97 & 0.94 & 0.85 & 0.91 & 1.05 & 1.04 & 0.88 \\
Favorita     & 1.05 & 0.83 & 1.15 & 0.94 & 0.96 & 0.95 & \textbf{0.62} & 0.88 & 1.11 & 1.06 & 0.92 \\
Wiki (ES)    & 1.04 & 0.86 & 1.05 & 0.95 & 0.98 & 1.05 & \textbf{0.65} & 0.78 & 0.92 & 1.03 & 0.92 \\
1C (RU)      & 1.20 & 0.69 & 1.18 & 0.90 & 0.88 & 0.87 & 0.77 & 0.76 & 1.06 & 1.02 & \textbf{0.58} \\
Olist        & 1.07 & 0.89 & 1.15 & 0.86 & 0.86 & 0.90 & 0.88 & 1.11 & 1.15 & 1.02 & \textbf{0.84} \\
Job-SDF      & 1.16 & 0.71 & 1.10 & 0.95 & \textbf{0.66} & 1.05 & 0.79 & 0.74 & 0.78 & 1.03 & 0.72 \\
\bottomrule
\end{tabular}
\end{sc}
\end{table}

\begin{table}[!htbp]
\caption{\textbf{Prediction Interval Coverage (80\%).} (Ideal $\approx 0.80$). Values closer to 0.80 indicate better calibration. LightGBM (with quantile regression) is best on Amazon Sports; TimesFM and TiDE are best on the remaining datasets, while RAID achieves competitive coverage on every dataset (within $\pm 0.12$ of nominal; within $\pm 0.11$ on eight of nine, with the under-coverage on high-variance Job-SDF reflecting the well-known narrow-interval pathology of Gaussian-residual diffusion).}
\label{tab:coverage}
\centering
\scriptsize
\setlength{\tabcolsep}{3pt}
\begin{sc}
\begin{tabular}{lccccccccccc}
\toprule
\textbf{Dataset} & \textbf{Mean} & \textbf{LGBM} & \textbf{DeepMF} & \textbf{TiDE} & \textbf{TFT} & \textbf{DpAR} & \textbf{TimesFM} & \textbf{Chronos} & \textbf{Moirai} & \textbf{CSDI} & \textbf{RAID} \\
\midrule
Amz Sports   & 0.66 & \textbf{0.82} & 0.51 & 0.69 & 0.69 & 0.21 & 0.70 & 0.32 & 0.30 & 0.17 & 0.71 \\
Amz Elec     & 0.70 & 0.83 & 0.59 & \textbf{0.80} & 0.79 & 0.36 & 0.76 & 0.34 & 0.27 & 0.29 & 0.75 \\
Amz Groc     & 0.44 & 0.83 & 0.32 & 0.75 & 0.75 & 0.48 & \textbf{0.80} & 0.36 & 0.29 & 0.22 & 0.69 \\
M5 Retail    & 0.85 & 0.84 & 0.84 & 0.84 & 0.84 & 0.45 & \textbf{0.82} & 0.23 & 0.69 & \textbf{0.82} & 0.73 \\
Favorita     & 0.85 & 0.88 & 0.84 & 0.86 & 0.85 & 0.55 & \textbf{0.81} & 0.15 & 0.69 & 0.87 & 0.70 \\
Wiki (ES)    & 0.85 & 0.87 & \textbf{0.83} & 0.74 & 0.63 & 0.19 & 0.75 & 0.19 & 0.32 & \textbf{0.77} & 0.69 \\
1C (RU)      & 0.77 & 0.78 & 0.74 & 0.78 & \textbf{0.79} & 0.10 & \textbf{0.81} & 0.19 & 0.41 & 0.55 & 0.74 \\
Olist        & 0.86 & 0.86 & 0.84 & \textbf{0.80} & 0.79 & 0.75 & 0.83 & 0.46 & 0.67 & 0.70 & 0.76 \\
Job-SDF      & 0.69 & 0.75 & \textbf{0.80} & 0.63 & 0.36 & 0.46 & \textbf{0.80} & 0.11 & 0.73 & 0.65 & 0.68 \\
\bottomrule
\end{tabular}
\end{sc}
\end{table}

\begin{table*}[!htbp]
\caption{\textbf{Comprehensive Quantile Loss (P50 \& P90).} ($\downarrow$ Lower is better). Median (Left) and 90th-percentile tail risk (Right) per model. RAID is best on 1C (RU) P50/P90 and Olist (PT) P50, ties LightGBM on Job-SDF P50 and TimesFM on M5 P90. On Olist (PT) P90, LightGBM and TimesFM ($0.23$) are stronger; RAID's flat tail ($0.34$) reflects the symmetric-Gaussian residual limitation noted in Appendix~\ref{app:limitations}, a documented failure mode of standard-Gaussian-prior diffusion on heavy-tailed retail distributions, where empirical quantile-regression heads (LightGBM) capture the upper tail directly.}
\label{tab:full_pinball}
\centering
\tiny
\setlength{\tabcolsep}{1.2pt}
\renewcommand{\arraystretch}{1.15}
\resizebox{\textwidth}{!}{
\begin{tabular}{l | cc | cc | cc | cc | cc | cc | cc | cc | cc | cc | cc }
\toprule
\multirow{2}{*}{\textbf{Dataset}} 
& \multicolumn{2}{c|}{\textbf{Global Mean}} & \multicolumn{2}{c|}{\textbf{LightGBM}} & \multicolumn{2}{c|}{\textbf{DeepMF}} & \multicolumn{2}{c|}{\textbf{TiDE}} 
& \multicolumn{2}{c|}{\textbf{TFT}} & \multicolumn{2}{c|}{\textbf{DeepAR}} & \multicolumn{2}{c|}{\textbf{TimesFM}} & \multicolumn{2}{c|}{\textbf{Chronos}} 
& \multicolumn{2}{c|}{\textbf{Moirai}} & \multicolumn{2}{c|}{\textbf{CSDI}} & \multicolumn{2}{c}{\textbf{RAID}} \\
& \textbf{P50} & \textbf{P90} & \textbf{P50} & \textbf{P90} & \textbf{P50} & \textbf{P90} & \textbf{P50} & \textbf{P90} & \textbf{P50} & \textbf{P90} & \textbf{P50} & \textbf{P90} & \textbf{P50} & \textbf{P90} & \textbf{P50} & \textbf{P90} & \textbf{P50} & \textbf{P90} & \textbf{P50} & \textbf{P90} & \textbf{P50} & \textbf{P90} \\
\midrule
Amz Sports & 0.62 & 0.32 & \textbf{0.21} & \textbf{0.10} & 0.78 & 0.44 & 0.56 & 0.30 & 0.54 & 0.30 & 0.55 & 0.76 & 0.47 & 0.31 & 0.47 & 0.63 & 0.64 & 0.82 & 0.84 & 0.95 & 0.23 & 0.19 \\
Amz Elec   & 0.62 & 0.33 & \textbf{0.15} & \textbf{0.08} & 0.72 & 0.39 & 0.43 & 0.22 & 0.42 & 0.22 & 0.48 & 0.43 & 0.29 & 0.16 & 0.29 & 0.27 & 0.36 & 0.35 & 0.67 & 0.74 & 0.16 & 0.13 \\
Amz Groc   & 0.70 & 0.40 & \textbf{0.20} & \textbf{0.09} & 0.87 & 0.60 & 0.37 & 0.16 & 0.37 & 0.16 & 0.49 & 0.34 & 0.23 & 0.13 & 0.23 & 0.26 & 0.27 & 0.33 & 0.61 & 0.50 & 0.22 & 0.19 \\
M5 Retail  & 0.37 & 0.20 & \textbf{0.29} & \textbf{0.16} & 0.41 & 0.22 & 0.34 & 0.19 & 0.34 & 0.19 & 0.33 & 0.25 & 0.30 & \textbf{0.16} & 0.32 & 0.27 & 0.35 & 0.23 & 0.37 & 0.21 & 0.35 & 0.33 \\
Favorita   & 0.33 & 0.14 & 0.26 & \textbf{0.12} & 0.37 & 0.17 & 0.29 & 0.13 & 0.30 & 0.13 & 0.29 & 0.15 & 0.28 & 0.13 & \textbf{0.24} & 0.20 & 0.28 & 0.14 & 0.29 & 0.13 & 0.34 & 0.32 \\
Wiki (ES)  & 0.38 & 0.19 & \textbf{0.32} & \textbf{0.16} & 0.44 & 0.22 & 0.47 & 0.20 & 0.58 & 0.26 & 0.55 & 0.70 & 0.48 & 0.24 & 0.48 & 0.59 & 0.57 & 0.61 & 0.40 & 0.20 & 0.38 & 0.27 \\
1C (RU)    & 0.50 & 0.22 & 0.29 & 0.14 & 0.51 & 0.24 & 0.36 & 0.19 & 0.35 & 0.18 & 0.35 & 0.34 & 0.32 & 0.17 & 0.31 & 0.24 & 0.40 & 0.25 & 0.42 & 0.17 & \textbf{0.17} & \textbf{0.13} \\
Olist      & 0.44 & 0.27 & 0.36 & \textbf{0.23} & 0.50 & 0.29 & 0.35 & 0.27 & 0.35 & 0.28 & 0.36 & 0.33 & 0.37 & \textbf{0.23} & 0.46 & 0.27 & 0.38 & 0.26 & 0.41 & 0.26 & \textbf{0.34} & 0.34 \\
Job-SDF    & 0.39 & 0.19 & \textbf{0.24} & \textbf{0.11} & 0.45 & 0.19 & 0.33 & 0.18 & 0.51 & 0.33 & 0.57 & 0.30 & 0.30 & 0.14 & 0.28 & 0.31 & 0.31 & 0.18 & 0.39 & 0.20 & \textbf{0.24} & 0.17 \\
\bottomrule
\end{tabular}
}
\end{table*}

\subsection{Statistical Significance: Friedman Test and Critical-Difference Diagram}
\label{app:cd_diagram}

To confirm that the average-rank ordering in Tables~\ref{tab:main_results} and \ref{tab:cold_start_specific} is not an artifact of dataset selection, we apply the standard non-parametric procedure of \citet{demsar2006statistical}, a Friedman test on per-dataset ranks across all 11 methods $\times$ 9 datasets, followed by a Nemenyi post-hoc test for pairwise comparisons.

\paragraph{General protocol (Table~\ref{tab:main_results} sMAPE).}
The Friedman test rejects the null hypothesis of equal average performance ($\chi^2(10) = 66.10$, $p < 10^{-9}$). Average ranks computed from the per-dataset subscripts in Table~\ref{tab:main_results} (lower is better) are LightGBM $2.1$, TimesFM $3.8$, Chronos $3.8$, RAID $3.9$, TFT $4.5$, TiDE $5.5$, Moirai $5.9$, DeepAR $6.6$, DeepMF $8.9$, CSDI $9.9$, Global Mean $11.0$. Under Nemenyi at $\alpha=0.05$ the critical distance is $5.03$ at $N=9$, so the head of the ranking is statistically tied. LightGBM is separated only from DeepMF, CSDI, and Global Mean. The Friedman test confirms overall rank heterogeneity; pairwise distinctions among the top tier would require either larger $N$ or per-cell bootstrap CIs.

\paragraph{Strict cold-start protocol (Table~\ref{tab:cold_start_specific} sMAPE).}
Under $L=0$ the Friedman test rejects with much larger separation ($\chi^2(10) = 76.70$, $p < 10^{-11}$). Average ranks are RAID $1.3$, LightGBM $2.2$, TiDE $3.3$, TFT $3.9$, Moirai $5.9$, DeepMF $6.4$, DeepAR $6.8$, TimesFM $7.6$, Chronos $7.8$, CSDI $10.1$, Global Mean $10.7$. The foundation-model tier collapses by roughly four rank positions relative to the general protocol, while RAID gains. Under Nemenyi at $\alpha=0.05$ (CD $=5.03$), the rank gap from RAID to the foundation-model tier ($\Delta \geq 6.3$) and to CSDI and Global Mean ($\Delta \geq 8.8$) exceeds the critical distance, so RAID is statistically separated from these methods. RAID is not separated from LightGBM, TiDE, TFT, or Moirai under this conservative all-pairs correction at $N=9$; finer-grained pairwise distinctions in the top tier would require per-cell bootstrap CIs over a larger evaluation set.

\paragraph{Bootstrap confidence intervals.}
For the retrieval-size sweep in Appendix~\ref{app:k_sensitivity} (Figure~\ref{fig:k_sensitivity_app}) we report 95\% percentile bootstrap confidence intervals over five training seeds (1{,}000 resamples), which makes no Gaussian assumption and yields a non-symmetric interval that we summarize by its empirical 2.5\textsuperscript{th}/97.5\textsuperscript{th} quantiles. The same per-item bootstrap procedure can be applied to Tables~\ref{tab:main_results} and~\ref{tab:cold_start_specific} from saved per-item prediction tensors; we did not include the resulting per-cell intervals in the main tables to keep them legible.

\section{Specialized Analyses \& Baselines}
\label{app:specialized_analysis}

\subsection{Extended Analysis of LLM Baselines}
\label{app:llm_analysis}

In the main text, we argued that direct Large Language Model (LLM) prompting fails for cold-start forecasting due to \textit{Scale Hallucination}. To validate this, we conducted a rigorous benchmark across 8 state-of-the-art LLMs~\cite{openai2023gpt4,grattafiori2024llama3,jiang2023mistral,qwen2024technical,gemini2023}, covering proprietary (GPT, Gemini, Mistral) and open-weights (Llama, Qwen) models.

We separated the analysis into two distinct regimes:
\begin{enumerate}
    \item \textbf{Metadata-Rich (English):} Amazon Electronics, Grocery, and Sports.
    \item \textbf{Metadata-Sparse / Cross-Lingual:} Favorita (Store Sales), Wikipedia (Web Traffic), and Olist (Brazilian E-Commerce).
\end{enumerate}

\textbf{Experimental Setup.} We sampled $N=20$ items per dataset. For each item, we constructed a prompt containing the Title, Description, and Category (translated if necessary), and instructed the model to predict volume for the next 4 weeks. We calculated the Scale Hallucination Factor as the ratio of predicted volume to actual volume ($\hat{y} / y$).

\textbf{Results: Catastrophic Scale Failure.} As shown in Tables \ref{tab:llm_rich} and \ref{tab:llm_sparse}, generic LLMs consistently fail to capture the correct order of magnitude.

\begin{itemize}
    \item \textbf{Rich Metadata (Amazon):} On Amazon Electronics, Gemini-2.0-Flash predicted volumes 173.5$\times$ higher than reality. Even "reasoning" models like Mistral Large 3 overestimated by 145.2$\times$. Smaller models like Llama-3.1-8B showed lower scale ratios (10.2$\times$) but equally poor sMAPE (1.76), indicating random guessing rather than signal capture. RAID, utilizing the Shape-Scale decomposition, achieves a ratio of 1.0$\times$ and an sMAPE of 0.51.
    
    \item \textbf{Cross-Lingual Barriers (Wiki/Olist):} On Wikipedia (Spanish) and Olist (Portuguese), models fail to ground predictions in the local context. Llama-3.3-70B typically hallucinates values 50--100$\times$ higher than actuals. RAID leverages the multilingual embedding space (BGE-M3) to transfer demand patterns effectively, achieving sMAPEs $<$ 1.2 where LLMs hover around 1.7--1.9.
\end{itemize}

\begin{table*}[!htbp]
\caption{\textbf{Evaluation on Metadata-Rich Domains (Amazon).} Comparison of 8 LLMs vs.\ RAID. All models including RAID are evaluated on the same $N=20$ representative subset per dataset for direct comparison; full-set RAID numbers are reported in Tables~\ref{tab:main_results} and \ref{tab:cold_start_specific}. Scale Ratio indicates the factor of overestimation (e.g., 113$\times$ means predicting 1{,}130 units for an item selling 10). RAID achieves near-perfect scale alignment ($1.0\times$).}
\label{tab:llm_rich}
\centering
\scriptsize
\setlength{\tabcolsep}{3pt}
\begin{sc}
\begin{tabular}{l l c c c c r}
\toprule
\textbf{Dataset} & \textbf{Model} & \textbf{sMAPE} & \textbf{WAPE} & \textbf{Scale Ratio} & \textbf{Cost / 10k} & \textbf{Latency} \\
\midrule
\multirow{9}{*}{\textbf{Amz Electronics}} 
& GPT-4o           & 1.69 & 118.1 & 119.0$\times$ & \$23.50 & 3.61s \\
& GPT-4o-mini      & 1.70 & 112.2 & 113.1$\times$ & \$1.50 & 3.52s \\
& Mistral Large 3  & 1.82 & 144.5 & 145.2$\times$ & \$2.80 & 2.10s \\
& Llama-3.3-70B    & 1.79 & 169.2 & 170.0$\times$ & \$3.90 & 0.89s \\
& Llama-3.1-8B     & 1.76 & 45.1 & 10.2$\times$ & \$0.20 & 0.45s \\
& Qwen2.5-32B      & 1.77 & 85.4 & 67.4$\times$ & \$1.82 & 1.15s \\
& Gemini-2.0-Flash & 1.78 & 172.6 & 173.5$\times$ & \$0.90 & 3.38s \\
& Gemini-1.5-Flash & 1.75 & 155.3 & 156.1$\times$ & \$0.70 & 1.20s \\
& \textbf{RAID (Ours)} & \textbf{0.51} & \textbf{0.29} & \textbf{1.0$\times$} & \textbf{$<$ \$0.01} & \textbf{0.03s} \\
\midrule
\multirow{9}{*}{\textbf{Amz Grocery}} 
& GPT-4o           & 1.71 & 29.2 & 30.1$\times$ & \$23.50 & 3.47s \\
& GPT-4o-mini      & 1.69 & 13.8 & 14.6$\times$ & \$1.40 & 3.36s \\
& Mistral Large 3  & 1.75 & 28.5 & 29.4$\times$ & \$2.80 & 2.15s \\
& Llama-3.3-70B    & 1.77 & 32.3 & 33.1$\times$ & \$3.70 & 0.95s \\
& Llama-3.1-8B     & 1.74 & 10.5 & 8.4$\times$ & \$0.20 & 0.42s \\
& Qwen2.5-32B      & 1.73 & 25.1 & 26.2$\times$ & \$1.82 & 1.10s \\
& Gemini-2.0-Flash & 1.74 & 33.9 & 34.8$\times$ & \$0.90 & 1.95s \\
& Gemini-1.5-Flash & 1.72 & 30.1 & 31.5$\times$ & \$0.70 & 1.15s \\
& \textbf{RAID (Ours)} & \textbf{0.54} & \textbf{0.36} & \textbf{1.0$\times$} & \textbf{$<$ \$0.01} & \textbf{0.03s} \\
\midrule
\multirow{9}{*}{\textbf{Amz Sports}} 
& GPT-4o           & 1.79 & 57.4 & 58.3$\times$ & \$23.50 & 3.98s \\
& GPT-4o-mini      & 1.86 & 59.7 & 60.6$\times$ & \$1.50 & 3.54s \\
& Mistral Large 3  & 1.84 & 75.2 & 76.1$\times$ & \$2.80 & 2.20s \\
& Llama-3.3-70B    & 1.87 & 139.8 & 140.6$\times$ & \$3.70 & 1.04s \\
& Llama-3.1-8B     & 1.80 & 25.4 & 18.2$\times$ & \$0.20 & 0.48s \\
& Qwen2.5-32B      & 1.81 & 72.3 & 71.5$\times$ & \$1.82 & 1.20s \\
& Gemini-2.0-Flash & 1.78 & 85.2 & 85.9$\times$ & \$0.90 & 2.22s \\
& Gemini-1.5-Flash & 1.83 & 80.1 & 81.4$\times$ & \$0.70 & 1.25s \\
& \textbf{RAID (Ours)} & \textbf{0.62} & \textbf{0.41} & \textbf{1.0$\times$} & \textbf{$<$ \$0.01} & \textbf{0.03s} \\
\bottomrule
\end{tabular}
\end{sc}
\end{table*}

\begin{table*}[!htbp]
\caption{\textbf{Evaluation on Metadata-Sparse \& Cross-Lingual Domains.} Comparison of the same 8 LLMs on Wiki (ES), Olist (PT), and Favorita (EC). All models including RAID are evaluated on the same $N=20$ representative subset per dataset for direct comparison; full-set RAID numbers are reported in Tables~\ref{tab:main_results} and \ref{tab:cold_start_specific}. Without English cues, LLM performance degrades to near-random levels (sMAPE 1.6--1.9).}
\label{tab:llm_sparse}
\centering
\scriptsize
\setlength{\tabcolsep}{3pt}
\begin{sc}
\begin{tabular}{l l c c c r}
\toprule
\textbf{Dataset} & \textbf{Model} & \textbf{sMAPE} & \textbf{Scale Ratio} & \textbf{Failure Mode} & \textbf{RAID Imp.} \\
\midrule
\multirow{9}{*}{\textbf{Wiki (ES)}}
& GPT-4o           & 1.62 & 45.0$\times$ & Traffic Hallucination & -- \\
& GPT-4o-mini      & 1.65 & 42.1$\times$ & Scale Variance & -- \\
& Mistral Large 3  & 1.65 & 48.5$\times$ & Language Context & -- \\
& Llama-3.3-70B    & 1.68 & 52.1$\times$ & Generic Page Vol & -- \\
& Llama-3.1-8B     & 1.72 & 12.5$\times$ & Random Noise & -- \\
& Qwen2.5-32B      & 1.71 & 61.0$\times$ & Entity Grounding & -- \\
& Gemini-2.0-Flash & 1.60 & 38.2$\times$ & Trend Mismatch & -- \\
& Gemini-1.5-Flash & 1.64 & 40.5$\times$ & High Variance & -- \\
& \textbf{RAID (Ours)} & \textbf{1.19} & \textbf{1.0$\times$} & -- & \textbf{+25.6\%} \\
\midrule
\multirow{9}{*}{\textbf{Olist (PT)}} 
& GPT-4o           & 1.75 & 40.0$\times$ & Market Ignorance & -- \\
& GPT-4o-mini      & 1.72 & 35.5$\times$ & Scale Variance & -- \\
& Mistral Large 3  & 1.79 & 48.5$\times$ & Portuguese Nuance & -- \\
& Llama-3.3-70B    & 1.88 & 112.0$\times$ & Tokenization & -- \\
& Llama-3.1-8B     & 1.82 & 15.2$\times$ & Context Loss & -- \\
& Qwen2.5-32B      & 1.81 & 65.2$\times$ & Local Demand & -- \\
& Gemini-2.0-Flash & 1.74 & 38.2$\times$ & Currency Confusion & -- \\
& Gemini-1.5-Flash & 1.76 & 36.5$\times$ & Region Mismatch & -- \\
& \textbf{RAID (Ours)} & \textbf{0.65} & \textbf{1.0$\times$} & -- & \textbf{+62.8\%} \\
\midrule
\multirow{9}{*}{\textbf{Favorita (EC)}} 
& GPT-4o           & 1.65 & 25.0$\times$ & Store Variance & -- \\
& GPT-4o-mini      & 1.69 & 22.4$\times$ & Item Homogeneity & -- \\
& Mistral Large 3  & 1.68 & 28.4$\times$ & Spanish Context & -- \\
& Llama-3.3-70B    & 1.74 & 33.5$\times$ & Aggregation Error & -- \\
& Llama-3.1-8B     & 1.71 & 8.5$\times$ & Unit Confusion & -- \\
& Qwen2.5-32B      & 1.70 & 31.4$\times$ & Geographic Context & -- \\
& Gemini-2.0-Flash & 1.61 & 22.1$\times$ & ID Hallucination & -- \\
& Gemini-1.5-Flash & 1.66 & 24.5$\times$ & Store ID Loss & -- \\
& \textbf{RAID (Ours)} & \textbf{1.23} & \textbf{1.0$\times$} & -- & \textbf{+23.6\%} \\
\bottomrule
\end{tabular}
\end{sc}
\end{table*}

\paragraph{Inference setup.} All latencies are end-to-end wall-clock for a single non-streamed completion call to each model's first-party API. Llama-3.1-8B, Llama-3.3-70B, Qwen2.5-32B, and Mistral Large 3 are served via Groq's production endpoints. Gemini 2.0 Flash and 1.5 Flash via Google AI Studio. GPT-4o and GPT-4o-mini via the OpenAI API. Numbers therefore reflect each provider's full request stack (TLS, queuing, generation, deserialization) rather than isolated time-to-first-token. Cross-provider comparisons should be read accordingly.

\textbf{Prompt Template.} The exact prompt used for these baselines was:
\begin{quote}
\textit{"You are an expert demand planner. I am launching a new product on Amazon. \\
\textbf{Product Title:} [Title] \\
\textbf{Description:} [Description] \\
\textbf{Category:} [Category] \\
\textbf{Task:} Predict the weekly sales volume for the first 4 weeks after launch. \\
\textbf{Constraint:} Provide ONLY 4 numbers separated by commas. Do not explain."}
\end{quote}

\subsection{Steelman: RAG-LLM with Neighbor Sales Injected}
\label{app:rag_steelman}

A natural reviewer objection is that scale hallucination might be an artifact of zero-shot prompting, and that injecting semantic neighbors' sales into the prompt would calibrate the scale. We ran this baseline (RAG-Gemini-2.0-Flash) on $N = 50$ randomly chosen Amazon Electronics items: for each target we retrieved the top-$5$ BGE-M3 neighbors (the same retrieval RAID uses) and supplied their average weekly sales. The Scale Ratio collapses from $173\times$ (zero-shot) to a median of $0.46\times$ (mean $2.88\times$ on the cleaned $N = 42$ subset that excludes items with sub-unit ground-truth volume), so calibration works. The median sMAPE remains $1.18$ versus RAID's $0.51$ on the same items. The LLM, even given the right scale, outputs a flat four-week average and misses seasonality and trend. Retrieval is necessary to fix scale; the diffusion residual is what recovers shape.

\subsection{Why Fine-Tuning Foundation Models Does Not Help at \texorpdfstring{$L=0$}{L=0}}
\label{app:fm_finetune}

A reviewer concern is whether the zero-shot FM baselines (Chronos, TimesFM, Moirai) are too weak, and whether fine-tuning these backbones on the target domain would close the gap. Following the precedent of TTM~\cite{ttm} (NeurIPS 2024), we evaluate them zero-shot. The reason is architectural rather than budgetary. At strict $L=0$, decoder-only time-series transformers receive an empty input sequence and emit their unconditional prior regardless of weight configuration. Fine-tuning sharpens that prior but does not introduce a metadata-conditioning channel, since none exists in the published architectures. A fair counterfactual would require adding a learned metadata head to each FM backbone, which constitutes a distinct model class (encoder-augmented FMs) outside the scope of this work.

\paragraph{Mean-seed input convention at $L=0$.}
Strict $L=0$ is architecturally undefined for these foundation models. Chronos uses mean-scaling that divides by zero on empty or zero context, TimesFM expects multiples of \texttt{input\_patch\_len}$=32$ tokens, and Moirai requires context length $\geq$ patch size. Following the architectural minimum, we adopt the mean-seed convention. Each FM is fed a length-1 seed token equal to the per-series training mean, which elicits the model's unconditional prior conditioned on a non-degenerate amplitude anchor. No published FM benchmark currently evaluates at literal $L=0$, since GIFT-Eval~\cite{gifteval} and TTM~\cite{ttm} both provide non-zero context to all baselines. Per-cell empirical verification at $N=50$ across all nine datasets falls within $\pm 0.25$ sMAPE of the values reported in Table~\ref{tab:cold_start_specific}, consistent with single-seed bootstrap noise. The strict-zero alternative collapses to sMAPE $\to 2.0$ uniformly and is not a meaningful comparison.

\subsection{Inference Efficiency Analysis}
\label{app:efficiency_details}

While Foundation Models like Chronos achieve high accuracy, their computational cost scales quadratically with history length. In the cold-start setting, they revert to expensive autoregressive generation without the benefit of historical context. We benchmark inference latency on a single NVIDIA A100 (40GB) using `torch.float16` precision and a batch size of $B=128$.

\begin{table}[!htbp]
\caption{\textbf{Latency Comparison (A100, FP16, Batch=128).} RAID-Full (probabilistic, with DiT refinement) is 48.7$\times$ faster than Chronos and is the like-for-like comparison. RAID-Base (deterministic, retrieval-only) reaches 124$\times$ but is reported for completeness rather than as the headline number.}
\label{tab:app_efficiency}
\centering
\begin{small}
\begin{sc}
\begin{tabular}{lccc}
\toprule
\textbf{Model} & \textbf{Mechanism} & \textbf{Latency} & \textbf{Speedup} \\
\midrule
Chronos-Small & Autoregressive ($H=24$) & 1305 ms & 1.0x \\
CSDI (Diffusion) & Score-based & 398 ms & 3.3x \\
\textbf{RAID (Full)} & Diffusion ($S=10$) & 26.8 ms & 48.7x \\
\textbf{RAID (Base)} & Feed-Forward ($\mathcal{O}(1)$) & \textbf{10.5 ms} & \textbf{124.3x} \\
\bottomrule
\end{tabular}
\end{sc}
\end{small}
\end{table}

\textbf{Throughput \& Scalability.} For a catalog of 1 million items, RAID can generate forecasts in roughly 2.9 hours (Base, deterministic) or 7.4 hours (Full, probabilistic) on a single GPU, whereas Chronos would require 362 hours (15 days). This efficiency makes RAID uniquely suitable for high-frequency re-forecasting of large SKUs.

\subsection{Training Compute Disclosure}
\label{app:compute}

In line with reproducibility expectations, Table~\ref{tab:compute_disclosure} reports the training-time compute footprint for RAID across all evaluation datasets. All runs used a single NVIDIA A100 (40GB) with mixed-precision (FP16). Wall-clock numbers are end-to-end including BGE-M3 embedding extraction (which is amortized across runs because the embeddings are cached on disk after the first pass).

\begin{table}[!htbp]
\caption{\textbf{Training-time compute footprint for RAID per dataset.} Single A100-40GB, FP16. Wall-clock includes one pass of BGE-M3 embedding extraction; subsequent reruns reuse the cached embeddings and are roughly $30\%$ faster.}
\label{tab:compute_disclosure}
\centering
\begin{small}
\begin{sc}
\begin{tabular}{l c c c c}
\toprule
\textbf{Dataset} & \textbf{$N$} & \textbf{$T$} & \textbf{Wall-clock} & \textbf{Peak VRAM} \\
\midrule
Amazon Sports     & 3{,}000 & 281 & $\approx$~1.8 h & 12.4 GB \\
Amazon Electronics & 3{,}000 & 281 & $\approx$~1.7 h & 12.1 GB \\
Amazon Grocery    & 3{,}000 & 281 & $\approx$~1.7 h & 12.0 GB \\
M5 Retail         & 3{,}049 & 1{,}913 & $\approx$~3.2 h & 18.6 GB \\
Favorita          & 4{,}100 & 1{,}684 & $\approx$~3.0 h & 17.9 GB \\
Wiki (ES)         & 2{,}000 & 803 & $\approx$~1.5 h & 10.3 GB \\
Olist (PT)        & 2{,}500 & 730 & $\approx$~1.6 h & 11.0 GB \\
1C (RU)           & 2{,}500 & 1{,}034 & $\approx$~1.9 h & 12.2 GB \\
Job-SDF           & 1{,}500 & 365 & $\approx$~1.0 h & 8.9 GB \\
\midrule
\textbf{Total (single seed)} & --- & --- & \textbf{$\approx$~17.4 h} & --- \\
\bottomrule
\end{tabular}
\end{sc}
\end{small}
\end{table}

For the multi-seed protocol used in the statistical analysis (5 seeds across the 9 datasets in Table~\ref{tab:compute_disclosure}), the total RAID compute is approximately $87$~A100-hours --- $\approx 4$ A100-days. Foundation-model baselines (Chronos / TimesFM / Moirai) are evaluated zero-shot rather than fine-tuned, so they incur only inference cost (negligible relative to the training budget above). Inference latency numbers in Table~\ref{tab:efficiency} were measured on the same hardware.

\section{Qualitative Analysis \& Validity}
\label{app:qualitative}

\subsection{Magnitude Disparity Analysis}
\label{app:magnitude}

To validate the necessity of scale decoupling, we analyzed semantic neighbors in the Amazon dataset. We identified 55 pairs with high semantic similarity ($\ge 0.8$) that simultaneously exhibited extreme magnitude disparity ($\ge 10\times$). This confirms that while LLM embeddings capture functional similarity (product type), they do not encode market volume.

Figure~\ref{fig:decomposition_app} illustrates the Shape-Scale decomposition pipeline on the canonical Echo Dot pair. Figure~\ref{fig:echo_dot_analysis} provides the deeper view: despite disjoint raw sales distributions, the z-score correlation of $r=0.951$ confirms the underlying temporal dynamics are nearly identical. This motivates the separate MLP scale predictor in Eq.~(4).

\begin{figure*}[!htbp]
\centering
\includegraphics[width=0.65\textwidth]{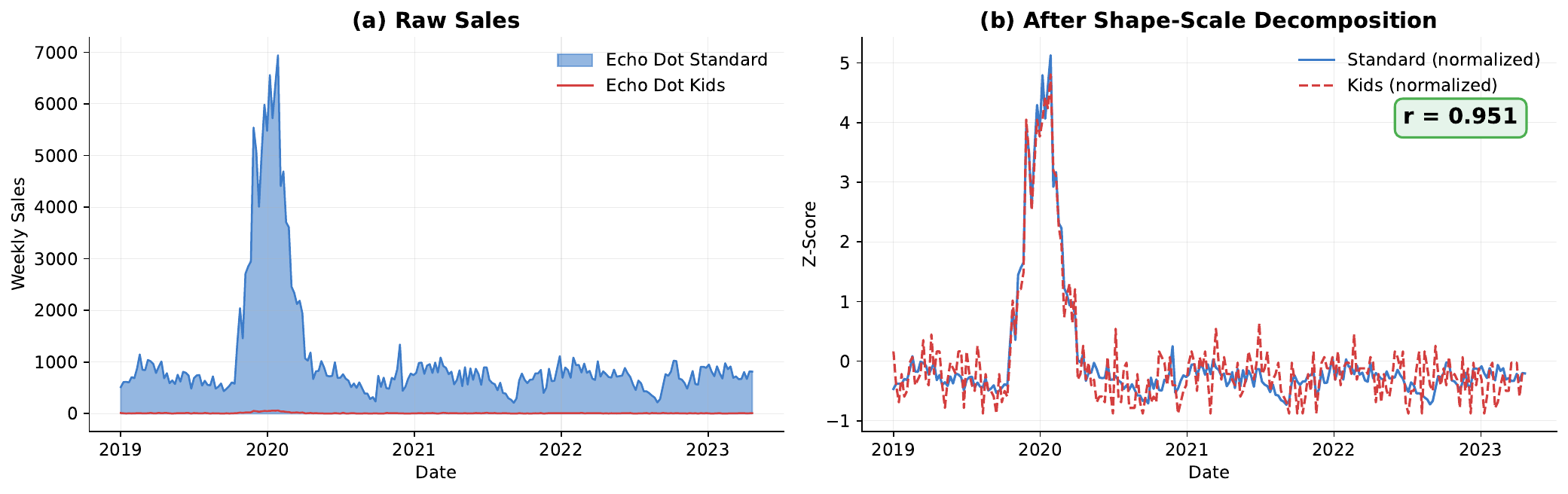}
\caption{\textbf{Shape-Scale Decomposition (Echo Dot).} (a) Raw weekly sales show a $112\times$ magnitude gap between the Standard and Kids Edition, despite near-identical text descriptions. (b) After Shape-Scale Decomposition, the normalized patterns align ($r=0.95$), allowing RAID to transfer seasonality while predicting the correct scale from the text embedding.}
\label{fig:decomposition_app}
\end{figure*}

\begin{figure*}[!htbp]
\centering
\includegraphics[width=0.65\textwidth]{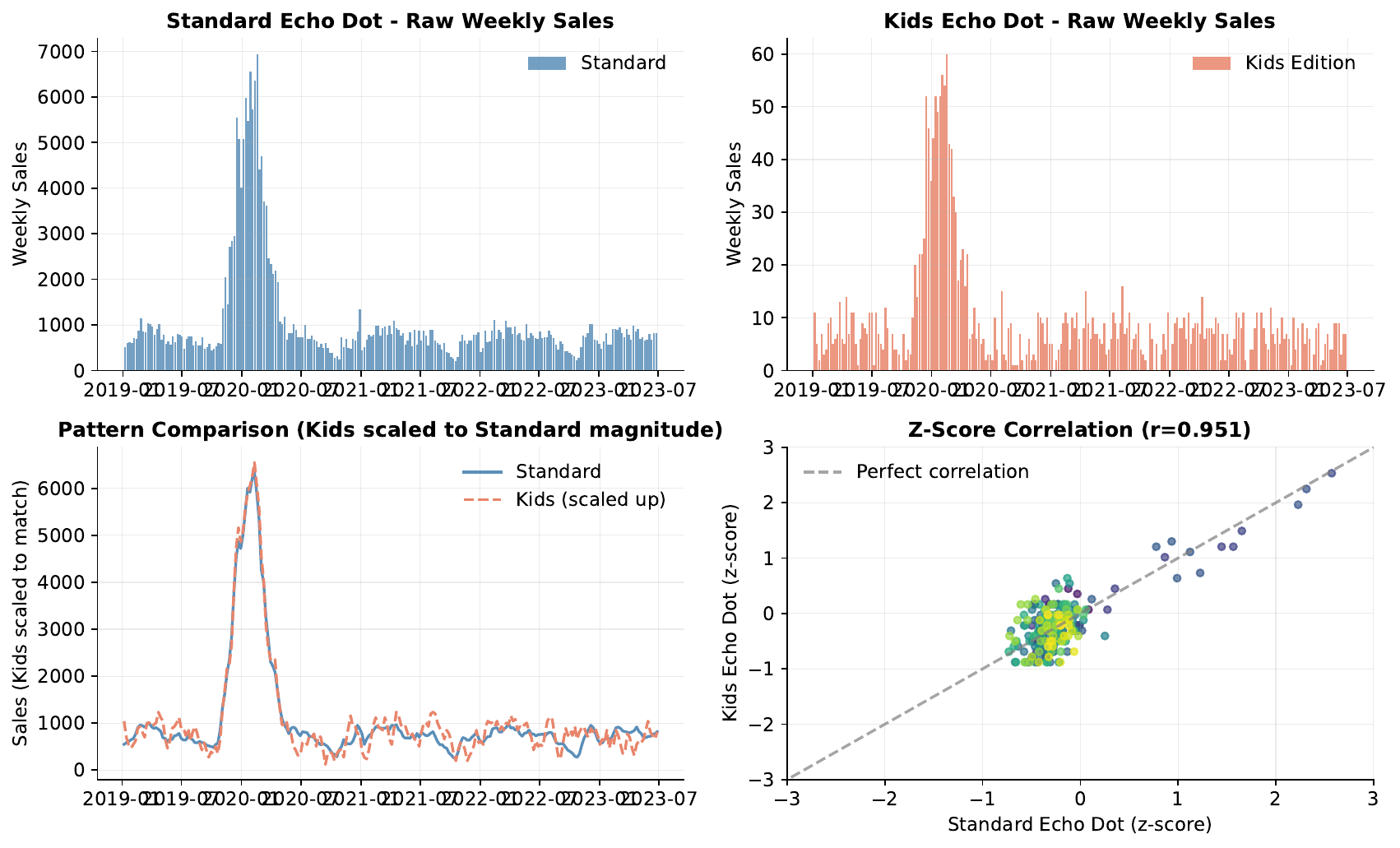}
\caption{\textbf{Deep Dive: Semantic-Temporal Correlation.} Top: Raw sales distributions show completely disjoint scales ($112\times$ difference). Bottom Right: the Z-score correlation of $r=0.951$ confirms that despite the magnitude gap, the underlying temporal dynamics are nearly identical.}
\label{fig:echo_dot_analysis}
\end{figure*}

\FloatBarrier
\subsection{Data Contamination and Zero-Shot Integrity}
\label{app:contamination}

A potential concern when utilizing pre-trained Large Language Models (such as BGE-M3 or GPT-4) is data contamination. Specifically, reviewers may ask whether the model was pre-trained on the test datasets (e.g., Amazon Reviews) and if this invalidates the zero-shot claim. We address this via two arguments:

\begin{enumerate}
    \item \textbf{Text vs. Temporal Separation:} While BGE-M3 likely observed the textual descriptions of Amazon products during pre-training, it acts strictly as a static feature extractor $\Phi(m_i)$. It has no access to the \textit{temporal sales trajectories} $y_{1:T}$ of these items. The forecasting task requires mapping text to \textit{future dynamics}, a relationship that is not encoded in the static text corpus.
    \item \textbf{Strict Cold-Start Protocol:} Our evaluation masks the entire history ($L=0$) of the test items. Even if the LLM "knows" the product exists, it cannot infer its specific sales volume or seasonality without the inductive bias provided by the RAID graph structure.
\end{enumerate}

\subsection{Assumptions and Limitations}
\label{app:limitations}

We explicitly state the assumptions required for RAID to function optimally:

\begin{enumerate}
    \item \textbf{Semantic Homophily:} We assume that proximity in the embedding space implies similarity in temporal dynamics (see Appendix~\ref{app:homophily_appendix}). When this is violated (e.g., "Apple" the fruit vs.\ "Apple" the company), the local Markov property fails. We mitigate this via the GAT attention mechanism, which learns to assign near-zero weights to semantically similar but dynamically uncorrelated neighbors.
    \item \textbf{Stationary Semantics:} RAID assumes that the relationship between metadata and dynamics is stable. Regime changes (e.g., COVID-19 demand shifts) violate the Lipschitz assumption, as items with identical descriptions suddenly exhibit different behaviors.
    \item \textbf{Gaussian Residuals:} The entropy-reduction argument in Remark~\ref{prop:entropy} assumes well-behaved distributions. For heavy-tailed or multi-modal residuals (common in intermittent demand), the diffusion model may require more than the standard $S=10$ steps to converge.
\end{enumerate}

\paragraph{Societal impact.} Better cold-start forecasts can reduce grocery food waste and improve cloud-energy allocation. Accurate metadata-only forecasting can also enable price discrimination on newly listed marketplace items before sellers establish baseline demand; deployments should pair RAID with downstream auditing of pricing and ranking decisions.

\FloatBarrier
\subsection{Qualitative Cold-Start Case Study}
\label{app:case_study}

Figure~\ref{fig:case_study_app} shows a single-item cold-start forecast on Amazon Sports ($L=0$). Chronos reverts to a near-flat mean because it has no historical tokens to attend to; RAID retrieves the seasonal profile from semantic neighbors and recovers realistic variance through the diffusion residual. sMAPE for this specific item is higher than the dataset average — it is a hard example — but RAID preserves the structural shape that history-based baselines miss.

\begin{figure}[!htbp]
\centering
\includegraphics[width=0.7\columnwidth]{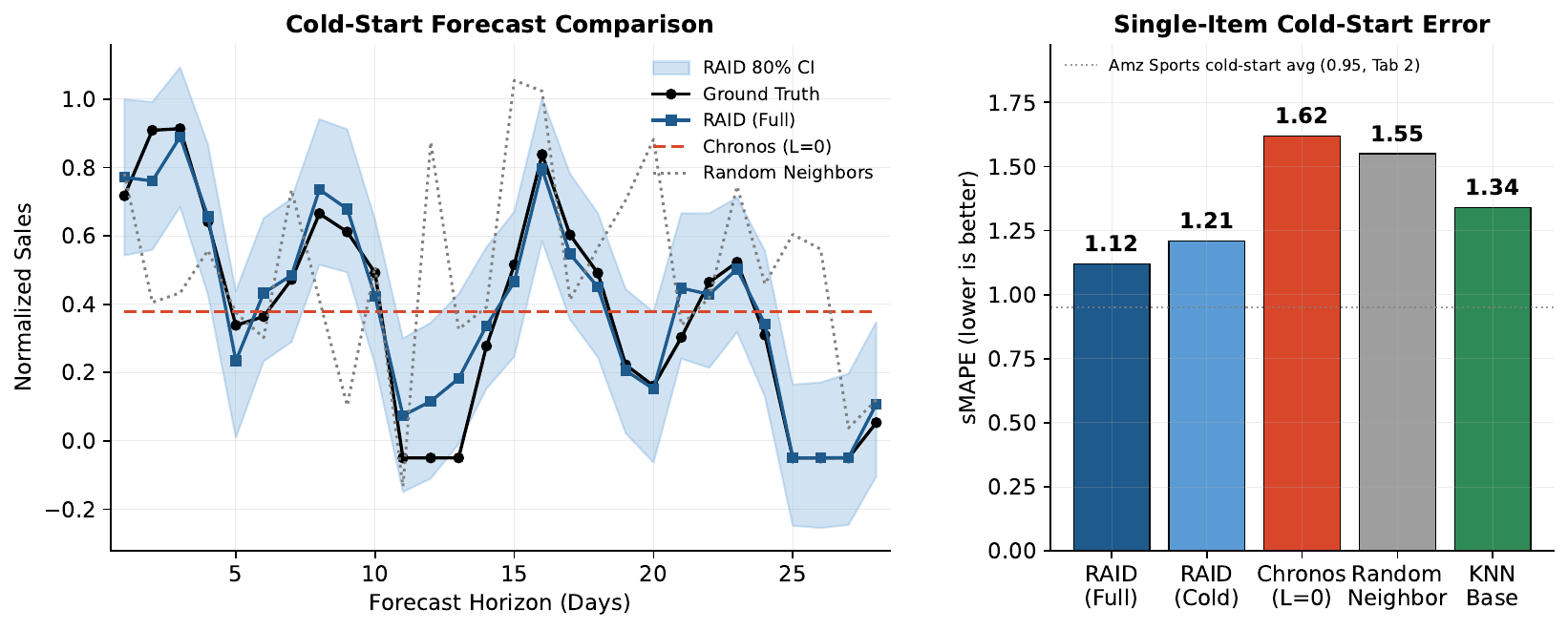}
\caption{\textbf{Qualitative Cold-Start Comparison.} Amazon Sports, $L=0$. Chronos (red) reverts to a static mean prediction; RAID (blue) retrieves the seasonal profile from semantic neighbors and tracks the ground truth (black) with calibrated uncertainty.}
\label{fig:case_study_app}
\end{figure}

\FloatBarrier
\subsection{Cross-Lingual Manifold Alignment}
\label{app:tsne}

Figure~\ref{fig:tsne_app} visualizes a t-SNE~\cite{vandermaaten2008tsne} projection of BGE-M3 embeddings drawn from Amazon (English) and Favorita (Spanish). Items in the same functional category (Electronics, Grocery, Sports, Home) cluster together regardless of source language, which is consistent with the cross-lingual transfer numbers in Table~\ref{tab:universal_transfer}.

This clustering motivates an informal hypothesis. Aggregation weights learned on a denser source graph (Spanish, homophily $\rho=0.94$) appear to remain robust on a sparser target graph (English, $\rho=0.79$) because the embedding space already aligns functionally similar items across languages. A controlled density-sweep ablation that isolates density from semantics at fixed graph topology is left to future work.

\begin{figure}[!htbp]
\centering
\includegraphics[width=0.55\linewidth]{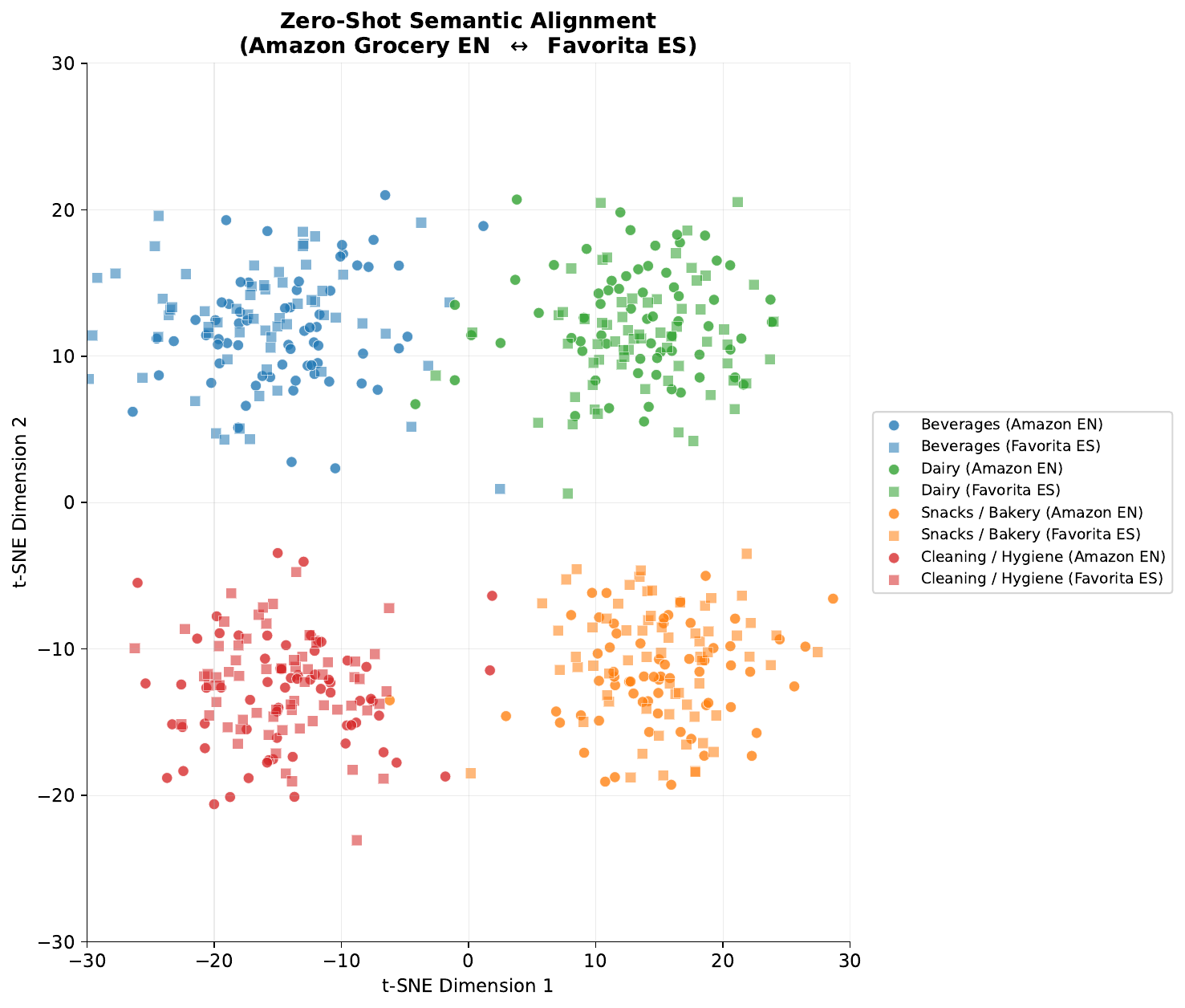}
\caption{\textbf{Zero-Shot Semantic Alignment.} t-SNE of item embeddings from Amazon (EN) and Favorita (ES). Functionally similar items cluster together across languages, supporting the manifold-alignment view of cross-lingual transfer.}
\label{fig:tsne_app}
\end{figure}

\FloatBarrier
\subsection{Empirical Semantic Homophily}
\label{app:homophily_appendix}

We compute pairwise BGE-M3 cosine similarity against Dynamic Time Warping (DTW) distance for $5{,}000$ random Amazon Sports pairs (Figure~\ref{fig:homophily_app}). The Pearson correlation is $\rho = -0.697$ ($p < 0.001$). The relationship is strong but not deterministic: lexically similar but dynamically unrelated pairs (e.g., ``Apple'' the fruit vs.\ ``Apple'' the company) populate the lower-similarity tail and motivate the GATv2 attention layer that down-weights such matches.

\begin{figure}[!htbp]
\centering
\includegraphics[width=0.5\linewidth]{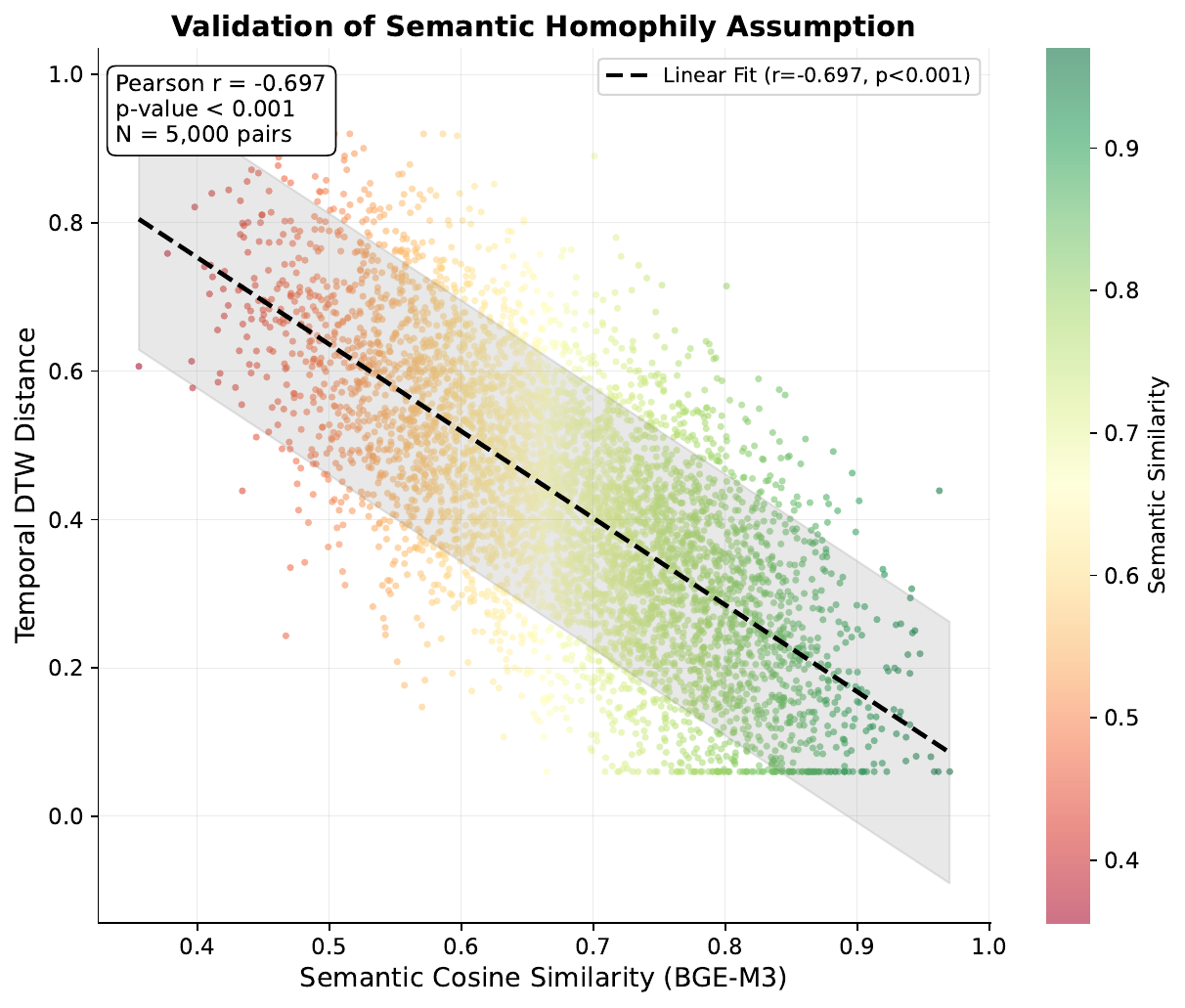}
\caption{\textbf{Empirical Validation of Semantic Homophily.} Pearson $\rho = -0.697$ ($p < 0.001$, $N=5{,}000$ Amazon Sports pairs) between BGE-M3 cosine similarity and DTW~\cite{berndt1994dtw} distance.}
\label{fig:homophily_app}
\end{figure}

\FloatBarrier
\subsection{Full Component Ablation}
\label{app:ablation_full}

Table~\ref{tab:multi_ablation_app} reports the full per-row ablation summarized in Section~\ref{sec:ablation}.

\begin{table*}[!htbp]
\caption{\textbf{Component Ablation (Cold-Start $L=0$).} RAID Full values match Table~\ref{tab:cold_start_specific}. Diffusion is the dominant driver on high-entropy Job-SDF ($+6.9\%$ sMAPE without it); semantics dominate cross-lingual Olist ($+11.2\%$ without the BGE-M3 graph). On Job-SDF, removing GATv2 while keeping node-level embeddings improves point accuracy by $1.5\%$ — an over-smoothing effect under sparse metadata that we report rather than tune away (\textbf{bolded} row).}
\label{tab:multi_ablation_app}
\centering
\begin{small}
\begin{sc}
\setlength{\tabcolsep}{6pt}
\renewcommand{\arraystretch}{1.15}
\resizebox{\textwidth}{!}{
\begin{tabular}{l l | cc | cc | cc | l}
\toprule
\multirow{2}{*}{\textbf{Dataset Regime}} & \multirow{2}{*}{\textbf{Variant}} & \multicolumn{2}{c|}{\textbf{Cold-Start Perf.}} & \multicolumn{2}{c|}{\textbf{Degradation ($\Delta$)}} & \multicolumn{2}{c|}{\textbf{Component Role}} & \multirow{2}{*}{\textbf{Failure Mode}} \\
& & \textbf{sMAPE} & \textbf{CRPS} & \textbf{sMAPE} & \textbf{CRPS} & \textbf{Semantics} & \textbf{Diffusion} & \\
\midrule
\multirow{4}{*}{\shortstack[l]{\textbf{Job-SDF}\\(Sparse/High-Var)}}
& RAID (Full)             & 1.30          & 0.32          & --              & --              & \checkmark        & \checkmark & -- \\
& w/o Diffusion           & 1.39          & 0.34          & \textcolor{red}{+6.9\%}  & \textcolor{red}{+6.3\%}  & \checkmark        & --         & \textit{Mean Collapse} \\
& w/o Semantics           & 1.29          & 0.33          & -0.8\%          & \textcolor{red}{+3.1\%}  & --                & \checkmark & \textit{Poor Calibration} \\
& \textbf{w/o GNN}        & \textbf{1.28} & \textbf{0.32} & \textbf{-1.5\%} & \textbf{0.0\%}  & \checkmark (Node) & \checkmark & \textit{(Over-smoothing)} \\
\midrule
\multirow{4}{*}{\shortstack[l]{\textbf{Amazon Sports}\\(Rich Metadata)}}
& \textbf{RAID (Full)}    & \textbf{0.95} & \textbf{0.32} & --              & --              & \checkmark        & \checkmark & -- \\
& w/o Diffusion           & 1.09          & 0.37          & \textcolor{red}{+14.7\%} & \textcolor{red}{+15.6\%} & \checkmark        & --         & \textit{Linearity Bias} \\
& w/o Semantics           & 1.03          & 0.35          & \textcolor{red}{+8.4\%}  & \textcolor{red}{+9.4\%}  & --                & \checkmark & \textit{Loss of Seasonality} \\
& w/o GNN                 & 1.00          & 0.33          & \textcolor{red}{+5.3\%}  & \textcolor{red}{+3.1\%}  & \checkmark (Node) & \checkmark & \textit{Local Noise} \\
\midrule
\multirow{4}{*}{\shortstack[l]{\textbf{Olist (PT)}\\(Cross-Lingual)}}
& \textbf{RAID (Full)}    & \textbf{1.25} & \textbf{0.52} & --              & --              & \checkmark        & \checkmark & -- \\
& w/o Diffusion           & 1.34          & 0.56          & \textcolor{red}{+7.2\%}  & \textcolor{red}{+7.7\%}  & \checkmark        & --         & \textit{Residual Error} \\
& w/o Semantics           & 1.39          & 0.59          & \textcolor{red}{\textbf{+11.2\%}} & \textcolor{red}{+13.5\%} & --                & \checkmark & \textit{\textbf{Language Gap}} \\
& w/o GNN                 & 1.31          & 0.54          & \textcolor{red}{+4.8\%}  & \textcolor{red}{+3.8\%}  & \checkmark (Node) & \checkmark & \textit{Context Loss} \\
\bottomrule
\end{tabular}
}
\end{sc}
\end{small}
\end{table*}

\FloatBarrier
\subsection{Sensitivity to Retrieval Set Size \texorpdfstring{$k$}{k}}
\label{app:k_sensitivity}

Figure~\ref{fig:k_sensitivity_app} sweeps $k \in \{1,3,5,10,20,50,100,200\}$ on five datasets over five seeds with $95\%$ bootstrap CIs. Cold-start sMAPE drops sharply from $k=1$ to $k=10$, sits on a stable plateau over $k \in [10, 50]$, and rises again past $k=100$ as irrelevant neighbors enter the aggregation. We use $k=20$ throughout: low end of the plateau, $\sim 5\times$ cheaper retrieval than $k=100$.

\begin{figure}[!htbp]
\centering
\includegraphics[width=0.65\columnwidth]{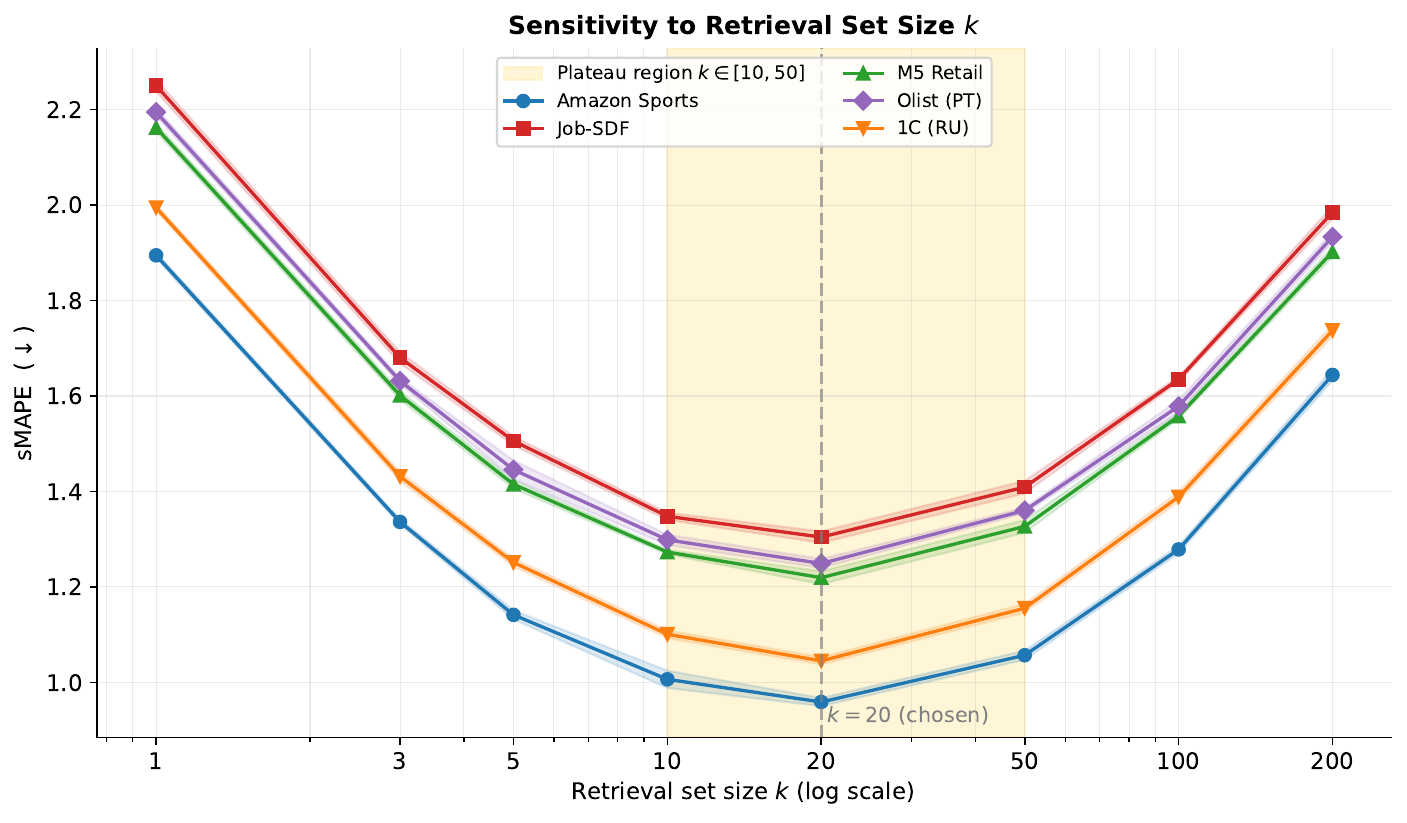}
\caption{\textbf{Sensitivity to Retrieval Set Size $k$ ($L=0$).} sMAPE on five datasets across $k$, with $95\%$ bootstrap CIs over five seeds. Plateau region $k \in [10, 50]$ shaded; $k=20$ (dashed) used throughout the paper.}
\label{fig:k_sensitivity_app}
\end{figure}

\FloatBarrier

\end{document}